\documentclass{article}

\usepackage[preprint]{arxiv}

\usepackage[utf8]{inputenc} % allow utf-8 input
\usepackage[T1]{fontenc}    % use 8-bit T1 fonts
\usepackage{hyperref}       % hyperlinks
\usepackage{url}            % simple URL typesetting
\usepackage{booktabs}       % professional-quality tables
\usepackage{amsfonts}       % blackboard math symbols
\usepackage{nicefrac}       % compact symbols for 1/2, etc.
\usepackage{microtype}      % microtypography
\usepackage{xcolor}         % colors

\usepackage[nolist]{acronym}
\usepackage{amsmath}
\usepackage{amssymb}
\usepackage{amsthm}
\usepackage{algorithm}
\usepackage[noend]{algpseudocode}
\usepackage{graphicx}
\graphicspath{{graphics/}}
\usepackage{cleveref}
\usepackage{paralist}
\usepackage{subcaption}
\usepackage{bm}
\usepackage{enumitem}
\usepackage{tikz}
\usepackage{thmtools}
\usepackage{thm-restate}
\usepackage{wrapfig}

\newcommand{\bc}[1]{\bm{\mathcal{#1}}}

\newcommand{\E}{\mathbb{E}}
\newcommand{\I}{\mathbb{I}}
\newcommand{\R}{\mathbb{R}}

\newcommand{\cA}{\mathcal{A}}
\newcommand{\cC}{\mathcal{C}}
\newcommand{\cD}{\mathcal{D}}
\newcommand{\cE}{\mathcal{E}}
\newcommand{\cH}{\mathcal{H}}
\newcommand{\cL}{\mathcal{L}}
\newcommand{\cO}{\mathcal{O}}
\newcommand{\cP}{\mathcal{P}}
\newcommand{\cS}{\mathcal{S}}
\newcommand{\cU}{\mathcal{U}}
\newcommand{\cV}{\mathcal{V}}
\newcommand{\cW}{\mathcal{W}}

\newcommand{\btau}{\bm{\tau}}
\newcommand{\bX}{\bm{\mathbf{X}}}
\newcommand{\bx}{\bm{\mathbf{x}}}
\newcommand{\bZ}{\bm{\mathbf{Z}}}
\newcommand{\bz}{\bm{\mathbf{z}}}

\newcommand{\bW}{\bm{\mathbf{W}}}
\newcommand{\bcC}{\bc{C}}
\newcommand{\bcD}{\bc{D}}
\newcommand{\bcH}{\bc{H}}
\newcommand{\bcG}{\bc{G}}

\newcommand{\pr}[1]{\left(#1\right)}
\newcommand{\br}[1]{\left[#1\right]}
\newcommand{\cbr}[1]{\left\{#1\right\}}
\newcommand{\abs}[1]{\left|#1\right|}
\newcommand{\doop}{do}
\newcommand{\eps}{\varepsilon}

\newcommand{\Unif}{\mathcal{U}nif}
\newcommand{\prob}{\mathbb{P}}

\algnewcommand{\Continue}{\textbf{continue}}
\algnewcommand{\Break}{\textbf{break}}
\algnewcommand{\multiline}[1]{\parbox[t]{\linewidth}{#1}}

\declaretheorem[name=Theorem,style=plain,numberwithin=section]{theorem}
\declaretheorem[name=Lemma,style=plain,sibling=theorem]{lemma}
\declaretheorem[name=Proposition,style=plain,sibling=theorem]{proposition}
\declaretheorem[name=Corollary,style=plain,sibling=theorem]{corollary}
\declaretheorem[name=Definition,style=definition,sibling=theorem]{definition}
\declaretheorem[name=Assumption,style=definition,sibling=theorem]{assumption}
\declaretheorem[name=Remark,style=remark,sibling=theorem]{remark}

\usetikzlibrary{arrows,positioning}
\tikzset{
    >=stealth',
    every edge/.append style={thick},
    node/.style={
        circle,
        rounded corners,
        draw=black,
        thick,
        minimum size=3em,
        text centered,
        scale=.7,
    },
}

\renewcommand{\cite}[1]{\citep{#1}}

\title{Causal Bandits without Graph Learning}

\author{%
  Mikhail Konobeev \\
  School of Computer and Communication Sciences, EPFL \\
  Borealis AI\\
  \texttt{mkon@hey.com} \\
  \And
  Jalal Etesami \\
  College of Management of Technology, EPFL \\
  Department of Computer Science, TUM \\
  \And
  Negar Kiyavash \\
  College of Management of Technology, EPFL
}

\begin{document}

\maketitle

\begin{abstract}
We study the causal bandit problem when the causal graph is unknown and develop
an efficient algorithm for finding the parent node of the reward node using
atomic interventions. We derive the exact equation for the expected number of
interventions performed by the algorithm and show that under certain graphical
conditions it could perform either logarithmically fast or, under more general
assumptions, slower but still sublinearly in the number of variables.
We formally show that our algorithm is optimal as it meets the universal lower
bound we establish for any algorithm that performs atomic interventions.
Finally, we extend our algorithm to the case when the reward node has multiple
parents.  Using this algorithm together with a standard algorithm from bandit
literature leads to improved regret bounds.
\end{abstract}

\begin{acronym}
\acro{dag}[DAG]{Directed Acyclic Graph}
\acro{mab}[MAB]{Multi-armed bandit}
\acro{pcm}[PCM]{Probabilistic Causal Model}
\acro{raps}[\textsc{raps}]{RAndomized Parent Search algorithm}
\acro{ucb}[UCB]{Upper Confidence Bound}
\end{acronym}

\section{Introduction}

\ac{mab} settings provide a rich theoretical context for formalizing and
analyzing sequential experimental design procedures.  Each arm in a \ac{mab}
setting represents an experiment/action and the consequence of pulling an arm
is represented by a stochastic reward signal. The objective of a learner in a
\ac{mab} problem is to select a sequence of arms  over a time horizon in order
to either find an arm that results in the maximum reward or to maximize the
cumulative reward during this time horizon.  Bandit problems have a growing
list of applications in various domains such as marketing
\cite{huo2017risk,sawant2018contextual}, recommendation systems
\cite{heckel2019active,silva2022multi}, clinical trials
\cite{liu2020reinforcement}, etc. \ac{mab} algorithms are designed for the
setting when there is no structural relationships between different arms.
However, this assumption is often violated in practice because of
interdependencies among the rewards of various arms.  To capture such
interdependencies, different structural bandit settings have been proposed such
as linear bandits \cite{abbasi2011improved}, contextual bandits
\cite{agrawal2013thompson,lattimore2020bandit}, and causal bandits
\cite{lattimore2016causal,lee2018structural} with the latter being the main
focus of this paper.

In causal bandit setting, the dependencies between the rewards of different
actions are captured by a causal graph and  actions are modeled as
interventions on variables of the causal graph \cite{lattimore2016causal}.
%As mentioned earlier, in this setting besides the effects  of several
%variables on the reward node, cause-effect relationships among these variables
%themselves, modeled by a causal graph, are also considered.
Causal bandits can effectively model complex real-world problems. For instance,
%in healthcare applications, dosages of multiple drugs can be adjusted
%adaptively to identify a desirable clinical outcome \cite{}. In
marketing strategists can adaptively adjust their strategy which can be modeled
as interventions made in their advertisement network to maximize revenue
\cite{nair2021budgeted,zhang2022causal}.

A major drawback of most existing work in causal bandit literature is  the
limiting assumption that the underlying causal graph is given upfront
\cite{lattimore2016causal}, which is frequently violated in most real-world
applications.  Similar to \citet{lu2021causal}, we also study the causal bandit
problem when the underlying causal graph is unknown.  However, unlike
\citet{lu2021causal} our work does not assume the knowledge of the essential
graph of the causal graph. Our main contributions are summarized as follows.
% notion of central node in the underlying causal graph that appears in special
% structures such as directed trees, causal forests
% \cite{greenewald2019sample}, or intersection-incomparable chordal graphs
% \cite{squires2020active}, our approach dose not relay on any structural
% assumptions.
%Furthermore, the theoretical guarantees provided in \cite{} requires
%structural assumptions for the underlying causal graph such as being a
%directed tree or causal forest, our approach does not relay on any
\begin{itemize}[leftmargin=*]
\item We propose (\cref{sec:algorithm}) and analyze (\cref{sec:analysis}) a
\acfi{raps} which does not assume the knowledge of the causal graph (or the
essential graph of the causal graph). In our analysis we derive \emph{the exact
equation} for the expected number of interventions performed by \ac{raps} on
any graph.
\item We describe two graphical conditions under which \ac{raps}
works in a fast or slow, but still sublinear in the number of nodes, regime
(\cref{sec:analysis,sec:multiparent}).
\item Based on \ac{raps} we propose a method that improves upon standard
bandit algorithm using causal structure of the arms and derive upper bounds for
the regret of this method (\cref{sec:regret}).
% \item We prove a universal lower bound on any algorithm that attempts to
% discover the parent node of the reward node with atomic interventions and show
% that \ac{raps} matches this bound exactly (\cref{sec:lower-bound}).
\end{itemize}

\section{Related Work}
In recent years, several work on Causal Bandit problem
\cite{lattimore2016causal,sen2017identifying,lee2018structural,nair2021budgeted,de2022causal}
have shown that  incorporation of causal structure improves upon the
performance of standard bandit \ac{mab} algorithms. However, the aforementioned
work relay on a limiting assumption that the underlying causal graph is given.
In this work, we remove this assumption.
%There are several approaches have been developed to relax this assumption.

When the causal graph is unknown, a natural approach is to first learn it
through observations and interventions.  Problem of learning a causal graph
from a mix of observations and interventions has been extensively studied in
causal structure learning literature
\cite{hauser2014two,hu2014randomized,shanmugam2015learning}. Yet learning the
entire underlying causal graph might not be necessary for a learner in order to
maximize its reward.  Further, merely learning the essential graph requires
more than linear (in terms of variables/nodes in the graph) number of
conditional independence tests~\cite{mokhtarian2022learning}.  Instead, we
propose an algorithm that discovers the parents of the reward node in sublinear
number of interventions on large classes of graphs. In case when it is known
that there is at most one parent of the reward node, all of these interventions
are atomic and we show that our algorithm is optimal.

\citet{de2022causal} propose a causal bandit algorithm which does not require
any prior knowledge of the causal structure and uses separating sets estimated
in an online fashion. Their theoretical result holds only when a true
separating set is known and the authors do not provide a final bound on the
regret.  The closest work to our paper  is that of \citet{lu2021causal} in
which the authors derive regret bounds for an algorithm based on central node
interventions. However, they assume the essential graph is known to the learner
while our algorithm makes no such assumption.

\section{Preliminaries}\label{sec:preliminaries}
A \acf{pcm} \cite{pearl2009causality} is a \ac{dag} $\bcG=(\cV,\cE)$
over a set of random variables $\cV$ with edges $\cE$ and a
distribution $\prob$ over the variables in $\cV$ that factorizes
with respect to $\bcG$ in the sense that the distribution over $\cV$
could be written as a product of conditional distributions of each
variable given its parents. We denote the number of vertices in
$\cV$ by $n$ and assume that each variable $X\in\cV$ takes value
from a finite set $[K]:=\cbr{1,\dots,K}$. The set of ancestors and
descendants of a node $X$ in $\bcG$ are denoted by $\cA_{\bcG}(X)$
and $\cD_{\bcG}(X)$, respectively. In both cases, we might omit
writing $\bcG$ when it is clear from the context.  In our definition
a node is its own ancestor and descendant and we will use horizontal
bar to exclude it, for example, for ancestors we will write
$\bar{\cA}(X)$ for $\cA(X)\setminus\cbr{X}$.  For a given subset
$\cS\subseteq\cV$, we define $\cA_{\bcG}(\cS):=\cup_{X\in\cS}\cA_{\bcG}(X)$
and $\cD_{\bcG}(\cS):=\cup_{X\in\cS}\cD_{\bcG}(X)$.  The vertex-induced
subgraph over nodes in $\cS$ is denoted by $\bcG_\cS$.  To simplify
the notation, we use $\cA_\cS(X)$ (similarly, $\cD_\cS(X)$) for the
set of ancestors (respectively, descendants) of $X$ in the induced
subgraph $\bcG_\cS$.  In addition, we will use superscript $c$ to
denote the non-ancestors/non-descendants, for example,
$\cA_\cS^c(X)=\cS\setminus\cA_\cS(X)$.  A collider on a path
$X_1,\dots,X_{\ell}$ between two nodes $X_1,X_\ell\in\cV$ is a node
$X_j$ with $1<j<\ell$ such that $X_j$ is a children of both $X_{j-1}$
and $X_{j+1}$, i.e., $X_{j-1}\to X_{j}\leftarrow X_{j+1}$.  For two
sets $A,B$, we denote their symmetric difference by $A\triangle
B:=(A\cup B)\setminus(A\cap B)$ and assume that all binary set
operations have the same precedence.

\subsection{Problem Setting}
In a causal bandit \cite{lattimore2016causal}, a learner $\cL$ performs a set
of interventions, i.e. actions, at each round $t\in[T]$ by setting a subset of
variables $\bX_t=(X_1,\dots,X_\ell)\subseteq\cV$ to some values
$\bx_t\in[K]^{\ell}$, denoted by $\doop(\bX_t=\bx_t)$. Playing the empty arm
denoted by $\doop()$ corresponds to observing a sample from the distribution
$\mathbb{P}$ underlying the \ac{pcm}.  The goal of the learner is to maximize a
designated reward variable $R$. When there is only one parent node of the
reward node in the graph $\bcG$ the causal bandit corresponds to standard
stochastic $K$-armed bandit. In what follows, we assume that the reward node
lies outside of the set of variables $\cV$, and thus we implicitly work with a
subgraph over the nodes $\cV\setminus\cbr{R}$.  We denote the parent set of the
reward variable by $\cP\subseteq \cV$.  In \cref{sec:algorithm,sec:analysis} we
start by assuming that $P$ is the only parent of $R$ and then later we
generalize our results to multiple parent nodes in \cref{sec:multiparent}. We
also allow for the reward node to have no parents in $\cV$ which we denote by
writing $P=\varnothing$.  The case when $P=\varnothing$ corresponds to having
an empty set of variables and thus we have
$\cA(\varnothing)=\cD(\varnothing)=\emptyset$.  The learner does not know the
underlying DAG over the variables in $\cV$ and cannot intervene directly on the
reward variable $R$.

Performance of a learner $\cL$ can be measured in terms of \emph{cumulative
regret} which
takes into account the rewards received from all the interactions performed,
\vspace{- .25cm}
\begin{align*}
R_\cL^T(\bcG,P)&=T\max_{\bX\subseteq\cV}
\max_{\bx\in[K]^{\abs{\bX}}}\E[R|\doop(\bX=\bx)]
-\sum_{t=1}^T\E[R|\doop(\bX_t=\bx_t)],
\end{align*}%\vspace{- .25cm}
or \emph{simple regret} which only focuses on the reward of the final intervention, predicted
to be the best by the learner after $T$ interactions, %\vspace{- .2cm}
\begin{align*}
r_\cL^T(\bcG,P)&=\max_{\bX\subseteq\cV}
\max_{\bX\in[K]^{\abs{\bX}}}\E[R|\doop(\bX=\bx)]
-\E[R|\doop(\bX_{T+1}=\bx_{T+1})],
\end{align*}
where, $\doop(\bX_{T+1}=\bx_{T+1})$ is the intervention
estimated to be the best by the learner $\cL$ after performing $T$
interactions and $|\bX|$ denotes the number of variables in $\bX$.%\vspace{- .5cm}

\paragraph{Remark.}
Note that in both definitions of regret, the learner is compared against an
oracle that always selects the best intervention. When the underlying DAG
$\bcG$ does not contain any unobserved variables, it is known that the best
intervention is always over the set of parent nodes of the reward node $R$
\cite{lee2018structural}.  Thus, in this work, we focus on a learner $\cL$ that
performs  interventions to detect the set of parent nodes of the reward node
and then finds the best assignment to $\cP$ in order to minimize regret.  Our
results in \cref{sec:analysis,sec:multiparent} can be used to bound both simple
and cumulative regret. For conciseness, we present only a cumulative regret
bound in the main text in \cref{sec:regret} and extend it to a simple regret
bound in \cref{sec:simple-regret}.

\section{Regret Analysis}\label{sec:regret}

In this section we present regret bounds achieved by a combination of our
algorithm aimed at discovering parent nodes and presented later in
\cref{sec:algorithm,sec:multiparent}, and a standard multi-armed bandit
algorithm such as \acs{ucb}~\cite{cappe2013kullback}. First, for simplicity we
assume that the reward variable is $[0,1]$-bounded although it is possible to
extend our results to more general $\sigma$-subgaussan variables.  Next, we
introduce the following assumptions which are similar to the assumptions in
\cite{lu2021causal}.
\begin{assumption}[Ancestoral Effect Identifiability]
\label{assumption:ancestoral-effect-identifiability} Let $\bZ\subseteq\cP$ be a
sequence of length $0\leq\ell\leq\abs\cP$ of last elements of $\cP$ in some
topological order. Further, let $X,Y\in\cV\setminus\cD(\bZ)$  be any two
variables such that $X\in\cA(Y)$ in $\bcG$. Assume $\abs{
\prob\cbr{Y=y|\doop(\bZ=\bz)}-\prob\cbr{Y=y|\doop(X=x,\bZ=\bz)}}>\eps$ for some
$x,y\in[K]$ and $\bz\in[K]^{\abs\bZ}$ where $\eps>0$ is a universal constant.
\end{assumption}
\begin{assumption}[Reward Identifiability]\label{assumption:reward-gap}
Let $X$ be an arbitrary ancestor of a node in $\cP$ in graph under intervention
over $\bZ$, $\bcG_{\overline{\bZ}}$, where $\bZ\subseteq\cP$ is a sequence of
length $0\leq\ell\leq\abs\cP$ of last elements of $\cP$ in some topological
order.  We assume that there exists $x\in[K]$ such that
$\abs{\E[R|\doop(\bZ=\bz)]-\E[R|\doop(X=x,\bZ=\bz)]}>\Delta$ for some
$\Delta>0$ and $\bz\in[K]^{\abs\bZ}$.
\end{assumption}
The first assumption allows our algorithm to obtain information about the
overall graph structure by only intervening on one node along with a subset of
the reward parents. The second assumption is necessary to determine whether the
reward node is a descendant of any node.  We hypothesize that
\cref{assumption:ancestoral-effect-identifiability} could be eased to only hold
for nodes $X,Y$ such that the shortest directed path between $X$ and $Y$ is at
most a certain length. In this case intervening on a node would provide
information about local structure of the graph. For the case when there is one
parent of the reward node and this information is available to the learner,
\citet{lu2021causal} show that the second assumption is necessary in that
without it any learner suffers $\Omega(\sqrt{nKT})$ regret in the worst case.

In order to bound the regret, we need to analyze the number of distinct nodes
intervened on by a learner $\cL$ in a graph $\bcG$ to find the set of parent
nodes $\cP$. We denote this quantity by $N_\cL(\bcG,\cP)$ and unless stated
otherwise, we assume that the learner uses the proposed \ac{raps} algorithm
presented first in \cref{alg:thealg} in \cref{sec:algorithm} for the single
parent case and later extend to multiple parents in \cref{sec:multiparent}. Our
regret bound is given for conditional regret defined as follows:
\begin{align*}
R_{\cL}^T(\bcG,P\mid E)
=T\max_{\bX\subseteq\cV}\max_{\bx\subseteq[K]^{\abs{\bX}}}\E[R|\doop(\bX=\bx)]
-\sum_{t=1}^T\E[R|\doop(\bX_t=\bx_t),E],
\end{align*}
where $E$ is the event that our algorithm correctly finds the set of parent
nodes, formally defined in \cref{lemma:batch-size}. Additionally, we denote by
$\Delta_{\bX=\bx}$ the mean reward gap of playing arm $\doop(\bX=\bx)$ for any
intervention set $\bX$ and realization $\bx$ from playing the best arm. Our
main result is the following \lcnamecref{thm:regret} proved in
\cref{appendix:regret}.
\begin{restatable}{theorem}{regret}\label{thm:regret}
Assume that $\cP\neq\emptyset$, i.e., the reward variable has at least one
parent in $\cV$. For the learner that uses \cref{alg:multiparent} and then runs
a \acs{ucb} the following bound\footnote{$f(n)\preceq g(n)$ stands for an
inequality up to a universal constant.} for the conditional regret holds with
probability at least $1-\delta$:
\begin{align}
\hspace{-.25cm}R_{\cL}^T(\bcG,P\mid E)
\!\leq \!\max\cbr{\frac1{\Delta^2},\frac1{\eps^2}}\!
K^{\abs\cP+1}\E[N(\bcG,\cP)]\log\pr{\!\frac{nK^n}{\delta}\!}
\!+\hspace{-.35cm}\sum_{\bx\in[K]^{\abs{\cP}}}\hspace{-.2cm}\Delta_{\cP=\bx}
\pr{\!1+\!\frac{\log{T}}{\Delta_{\cP=\bx}^2}}.\label{eq:regret-bound}
\end{align}
% In particular, if $(\bcG,P)$ satisfy the conditions of
% \cref{thm:log-condition}, then the conditional regret is asymptotically
% bounded by
% \begin{align*}
% &\cO\Biggl(K\max\cbr{\frac{1}{\Delta^2},
% \frac{1}{\eps^2}}\log\pr{\frac{nK}{\delta}}\log(n)\\
% &\qquad+\sqrt{KT\log(T)}\Biggr)
% \end{align*}
% with probability at least $1-\delta$ and if $(\bcG,P)$ satisfy the
% condition of \cref{thm:sublinear}, then it is bounded by
% \begin{align*}
% &\cO\Biggl(K\max\cbr{\frac{1}{\Delta^2},
% \frac{1}{\eps^2}}\log\pr{\frac{nK}{\delta}}\frac{n}{\log_d(n)}\\
% &\qquad+\sqrt{KT\log(T)}\Biggr)
% \end{align*}
% with probability at least $1-\delta$ where $d$ is equal to the maximum degree
% in the skeleton of $\bcG$.
\end{restatable}
Our regret bound has two terms: the first comes from finding the set of parent
nodes and the second from determining the best intervention over the parents.
The bound above improves on the performance of standard multi-armed bandit
algorithm because the terms in \cref{eq:regret-bound} depend on the number of
parents of the reward node and there are $K^{\abs{\cP}}$ such parents, while
with standard bandit algorithm there will be $K^n$ terms in the summation
similar to the second term in \cref{eq:regret-bound}. The main limitation of
our work is the $\frac1{\min\cbr{\eps,\Delta}^2}$-dependence and we believe
that by playing each arm proportionally to its' inverse reward gap while
simultaneously trying to estimate the set of parent nodes is a good direction
for future work that would improve this dependence. In
\cref{appendix:experiments} we provide experimental results showing the values
of $\eps$ and $\Delta$ for different Erd\H{o}s-R\'enyi graphs. In what follows
we present and provide an analysis of our algorithm to discover parent nodes.

\section{Randomized Parent Search Algorithm}\label{sec:algorithm}
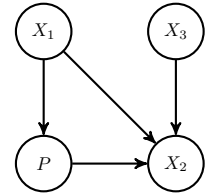
\begin{wrapfigure}{r}{0.3\linewidth}
\centering
\begin{tikzpicture}[->,every node/.style=node]
\node (x1) {$X_1$};
\node[below=of x1] (p) {$P$};
\node[right=of p] (x2) {$X_2$};
\node[right=of x1] (x3) {$X_3$};
\path
(x1) edge (p)
(x1) edge (x2)
(p) edge (x2)
(x3) edge (x2);
\end{tikzpicture}
\caption{An example of \ac{dag} with a single parent node $P$.}
\label{fig:example}
\end{wrapfigure}
In this section we present our learner, i.e., \acf{raps}, for the case when the
reward node has at most one parent. The algorithm is shown in
\cref{alg:thealg}.  We denote the parent node by $P$ (which is set to be
$\varnothing$ when there is no parent) and denote the number of distinct nodes
intervened on by our algorithm by $N(\bcG,P)$. This algorithm could be run
multiple times to discover each parent node as will be later discussed in
\cref{sec:multiparent}.  After the parent node is discovered, one could use
standard algorithms from bandit literature \citep[see, for
example,][]{lattimore2020bandit} to find the best intervention over it to
minimize the simple or cumulative regret.  \Cref{alg:thealg} defines a
recursive function \textsc{REC} with single argument denoted by $\cC$ --- the
so called \emph{candidate set} of nodes in $\bcG$ which might contain $P$ ---
and this function is called initially with all the nodes in the graph as its
argument.

\begin{algorithm}[tb]
\caption{\acf{raps} for single parent node}
\begin{algorithmic}[1]
\Require{Set of nodes $\cV$ of $\bcG$ given as input}
\Ensure{The parent node $P\in\cV$ of the reward node or $\varnothing$
if there is no parent node in $\cV$}
\State\ac{raps} works by calling
\Call{rec}{$\cC=\cV$} defined as follows
\Function{rec}{$\cC$}\label{line:func}
\If{$\cC=\emptyset$}
\State\Return $\varnothing$
\EndIf
\State $X\sim\Unif(\cC)$\label{line:sample}
\State Intervene on $X$ to determine if $P\in\cD_\cC(X)$
\label{line:intervention}
\If{$P\in\cD_\cC(X)$}\label{line:condition}
\State $\hat{P}\gets\text{\Call{rec}{$\cD_\cC(X)\setminus\cbr{X}$}}$
\label{line:ancestor-call}
\If{$\hat{P}=\varnothing$}
\State\Return$X$
\EndIf
\State\Return$\hat{P}$
\EndIf
\State\Return\Call{rec}{$\cC\setminus\cD_\cC(X)$}\label{line:not-ancestor-call}
\EndFunction
\end{algorithmic}
\label{alg:thealg}
\end{algorithm}

We will explain \ac{raps} first with an example. Consider the graph with four
nodes in \cref{fig:example} where $P$ is the parent of reward node $R$ (not
shown on the figure). The algorithm starts by calling the recursive function
with $\cC=\cV=\cbr{X_1,X_2,X_3,P}$. Assume that during this call the recursive
function samples $X_3$. Changing the value of this node should allow the
learner to determine the descendants which are in this case $\cbr{X_2,X_3}$ and
do not include $P$. The learner realizes this because $R$ does not change
unless P changes. Thus, there will be another call of the recursive function
with $\cC=\cV\setminus\cbr{X_2,X_3}=\cbr{X_1,P}$ on
\cref{line:not-ancestor-call}.  After that, if in the recursion the node $X_1$
is sampled, the same function is called on \cref{line:ancestor-call} with
$\cC=\cbr{P}$. This is because $P$ is the only descendant of $X_1$ not
including $X_1$ in the graph over the nodes in $\cbr{X_1,P}$. Lastly, the
algorithm will have to sample $P$ and return it as the discovered parent node.

\subsection{Determining Descendants}\label{sec:determine-ancestor}
In general, \ac{raps} intervenes on a randomly selected node $X\in\cC$ on
\cref{line:intervention}. Several interventions on $X$ should be sufficient to
determine the descendants of $X$ and whether $P\in\cD_{\cC}(X)$.  This is
because changing the value of $X$ should change the values of the descendants
of $X$ and we can determine if $P\in\cD_{\cC}(X)$ by checking if the value of
the reward variable $R$ changes. Notice that if $Y\not\in\cC$, then none of the
descendants of $Y$ are in $\cC$ which means that it is possible to find
$\cD_{\cC}(X)$ simply by taking $\cD_{\bcG}(X)\cap\cC$.

Let $\bar{R}$ and $\bar{R}^{\doop(X=x)}$ denote sample mean of the reward
variable under observational and interventional distributions.  Our
\cref{assumption:reward-gap} allows us to determine whether an arbitrary node
$X$ is an ancestor of the reward node.  This is done by comparing
$\abs{\bar{R}-\bar{R}^{\doop(X=x)}}$ for all $x\in[K]$ with $\Delta/2$ and
concluding that $X$ is an ancestor of $P$ in $\bcG$ (and therefore in
$\bcG_{\cC}$ where $\cC$ is an argument passed to the recursive function in
\cref{alg:thealg}) if for some $x\in[K]$ the absolute difference exceeds the
threshold. \Cref{assumption:ancestoral-effect-identifiability} allows to
determine the descendants of $X$ after an intervention on it.  For this we
consider as the descendants the set of nodes $Y\in\cC$ such that for some
$x,y\in[K]$ the absolute difference $\abs{\hat{P}(Y=y)-\hat{P}(Y=y|\doop(X=x)}$
exceeds $\eps/2$, where $\hat{P}(\cdot),\hat{P}(\cdot|\doop(X=x))$ are the
empirical distributions over $Y$ without any intervention and under
intervention $\doop(X=x)$. In \cref{lemma:batch-size} we provide the exact
number of times the algorithm needs to intervene on each node, i.e.  the sample
sizes to compute $\bar{R},\bar{R}^{\doop(X=x)},\hat{P}(\cdot),
\hat{P}(\cdot|\doop(X=x))$, in order to find the parent node with high
probability.

Our analysis in later sections only concern the number of distinct nodes our
algorithm needs to intervene on to find the parent node $P$.  In what follows,
we first present \emph{the exact expression} to compute the expected number of
distinct node interventions performed by \ac{raps}.  Next, we introduce classes
of \acp{dag} for which this expected value is either asymptotically logarithmic
or sublinear in the number of nodes $n$.

\section{Analysis of \ac{raps} for Single Parent Case}\label{sec:analysis}

We start by stating the exact expression for the expected number of distinct
node interventions performed by \cref{alg:thealg}. The proof is in
\cref{appendix:exact}.

\subsection{Expected Number of Interventions}\label{sec:exact}
\begin{restatable}{theorem}{exact}\label{thm:exact}
The expected number of distinct node interventions performed by a learner that
uses \cref{alg:thealg} to determine the parent node $P$ is given by
\begin{equation}
\E[N(\bcG,P)]=
\sum_{X\in\cV}\frac{1}{\abs{\cA(P)\triangle\cA(X)\cup\cbr{X}}}.
\label{eq:exact}
\end{equation}
\end{restatable}

Next, we present two conditions under which \ac{raps} performs sublinearly.  In
the ``fast'' regime, it requires $\cO(\log(n))$ expected number of
interventions, while in the ``slow'' regime, it requires
$\cO\pr{\frac{n}{\log_d(n)}}$ expected number of interventions with $d$ being
the maximum degree in the skeleton of $\bcG$.  It is noteworthy that our
algorithm even in the slow regime outperforms the na\"ive exploration method
that requires $\Omega(n)$ interventions.  Moreover, in
\cref{appendix:lower-bound}, we introduce a universal lower bound on the
expected number of interventions required by any learner to find the parent
node and show that \cref{eq:exact} matches this lower bound.

\subsection{Fast Regime}
In order to introduce the condition under which \ac{raps} performs fast, we
first characterize the candidate sets $\cC\subseteq \cV$, that is the sets of
nodes that the recursive function \cref{alg:thealg} could be called with as an
argument. To this end, we define the following family of subsets of $\cV$.
%We start by characterizing the candidate sets with which the recursive
%function on \cref{line:func} of \cref{alg:thealg} is called. Our condition for
%the fast performance of our algorithm is stated as a condition for all
%subgraphs $\bcG_{\cC}$ of large enough size with $\cC$ being an arbitrary
%candidate set with which the recursive function on \cref{line:func} of
%\cref{alg:thealg} is called.  We start by defining a candidate family of sets.
\begin{definition}
A \emph{candidate family} of a graph $\bcG$ with a parent node $P$ is a family
of subsets given by
\begin{align*}
&\bcC_{\bcG}(P):=\cbr{\cD^c(\cW)\big|\;\cW\subseteq\cA^c(P)}
\cup\cbr{\cD_{\cD(X)}^c(\cW)\setminus\cbr{X}\big|X\in\cA(P),
\cW\subseteq\cA_{\cD(X)}^c(P)}.
\end{align*}
\end{definition}
Let $\cW$ be an arbitrary set of non-ancestors of $P$. All descendants of these
non-ancestors could be removed from the starting candidate set $\cC=\cV$. This
corresponds to the first family on the right hand side in the definition of
$\bcC_{\bcG}(P)$. At the same time, the algorithm might also reduce the set of
candidate nodes if it discovers an ancestor of the parent node $P$. This
happens when the recursive function is called on \cref{line:ancestor-call}. Let
$X$ be an intervened on ancestor of $P$, then the candidate set reduces to the
subset of descendants of $X$. This set might again exclude arbitrary
non-ancestors of $P$ previously denoted by $\cW$, but this time in the subgraph
over $\cD(X)$. We provide an example of the candidate family for the line graph
in \cref{fig:line-graph} later in this section. Next
\lcnamecref{lemma:candidate-family} shows that when the recursive function in
\cref{alg:thealg} is called, its argument belongs to the candidate family
$\bcC_{\bcG}(P)$.  The proof is in \cref{appendix:log-bound}.

\begin{restatable}{lemma}{cfamily}\label{lemma:candidate-family}
All possible arguments $\cC$ with which the recursive function in
\cref{alg:thealg} is called are contained within the candidate
family $\bcC_{\bcG}(P)$.
\end{restatable}

At a high level, \cref{alg:thealg} performs $\cO(\log n)$ interventions if for
each $\cC\in\bcC_{\bcG}(P)$ of large size, the number of ancestors of $P$ in
$\bcG_{\cC}$ is large, or the number of non-ancestors of $P$ each of which has
large number of descendants is large. The latter condition could be interpreted
as the condition that the non-descendants of $P$ asymptotically form a line
graph.  This is captured formally by the following result, proved in
\cref{appendix:log-bound}.

\begin{restatable}{theorem}{logcond}\label{thm:log-condition}
For a constant $0<\alpha<1$ and $\cC\in\bcC$, let the set of ``heavy''
non-ancestors to be
\begin{align}
\cH_\cC(\alpha)&:=\cbr{X\in\cA_\cC^c(P)
\big|\; \abs{\cD_\cC(X)}\geq\alpha\abs{\cC}}\label{eq:logn-c1}.
\end{align}
Assume that $\bcG$ is such that for any $\cC\in\bcC_{\bcG}(P)$
at least one of the following holds
i) $\abs{\cH_\cC(\alpha)}\geq\beta\abs{\cC}$,
ii) $\abs{\cA_\cC(P)}\geq\gamma\abs{\cC}$, or
iii) $\abs\cC\leq c\log^k(n)$,
%\begin{compactitem}
%\item $\abs{\cH_\cC(\alpha)}\geq\beta\abs{\cC}$,
%\item $\abs{\cA_\cC(P)}\geq\gamma\abs{\cC}$,
%\item $\abs\cC\leq c\log^k(n)$,
%\end{compactitem}
for fixed $0<\alpha,\beta,\gamma<1$, $c\in\R_{>0}$, and $k\geq1$.
Then, $\E[N(\bcG,P)]=\cO(\log^k n).$
\end{restatable}
The assumption of \cref{thm:log-condition} states that for all
candidate sets considered by the recursive function in \cref{alg:thealg}
either the cardinality of $\cC$ is upper bounded by $c\log^k(n)$
or in the subgraphs $\bcG_\cC$ one of the following holds:
\begin{itemize}[leftmargin=*]
\item there is a $\beta$-fraction of nodes that are among the non-ancestors of
$P$ and have at least an $\alpha$-fraction of nodes as their descendants,
\label{item:line-condition}
\item there is a $\gamma$-fraction of nodes that are ancestors of $P$.
\end{itemize}

\begin{figure}
\centering
\begin{tikzpicture}[->,every node/.style=node]
\node (p) {$P$};
\node[right=of p] (x1) {$X_1$};
\node[right=of x1] (x2) {$X_2$};
\node[right=of x2,draw=none] (dots) {$\dots$};
\node[right=of dots,scale=0.9] (xnm1) {$X_{n-1}$};
\path
(p) edge (x1)
(x1) edge (x2)
(x2) edge (dots)
(dots) edge (xnm1);
\end{tikzpicture}
\caption{An example of a line graph, such graphs satisfy the
condition of \cref{thm:log-condition}.}\label{fig:line-graph}
\end{figure}
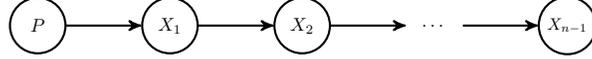

Notice that in the condition of \cref{thm:log-condition} there is no
restriction on the structure of the ancestors of $P$. Thus, if the number of
ancestors of $P$ is sufficient we have that \ac{raps} has at most logarithmic
number of distinct node interventions.  At the same time, when there are more
than constant number of non-ancestors, \cref{thm:log-condition} requires them
to have a certain structure.  Consider, for example, the line graph in
\cref{fig:line-graph}. In this case the candidate family consists of the sets
$\cbr{X_1,\dots,X_{i-1}}$ and $\cbr{P,X_1,\dots,X_{i-1}}$ for $i\in[n-1]$.  The
condition of \cref{thm:log-condition} still holds since for every candidate set
$\cC\in\bcC_{\bcG}(P)$ it holds that there is a $\beta$-fraction of ``heavy''
non-ancestors of $P$ in $\bcG_\cC$. Such non-ancestors contain many other
non-ancestors of $P$ as their descendants. To see this, let
$\cC=\cbr{X_1,\dots,X_{i-1}}$ for some $i\in[n-1]$ (the other case is similar)
and consider the set $\cbr{X_1,\dots,X_{\lfloor{i/2}\rfloor}}$. Each node in
this set has at least $i/2$ descendants and there are at least $i/2-1$ such
nodes.  Thus, the condition holds for $\alpha,\beta$ close to $\frac12$. Note
also that even if $P$ is not the first node in a topological ordering of a
graph but its' non-descendants still satisfy the condition of
\cref{thm:log-condition}, then the \ac{raps} succeeds in $\cO(\log^k n)$
expected number of interventions. Additionally, from \cref{thm:exact}
it is easy to see that the $\log{n}$ upper bound on the number of
distinct node interventions is tight.

At the same time, the condition of \cref{thm:log-condition} for non-descendants
in subgraphs over candidate sets of large size is more general than just
requiring all non-descendants to form a line graph. First, notice that if the
parent node has $\Omega(n)$ children, each with only one parent, then the
algorithm requires $\Omega(n)$ interventions no matter how the remaining nodes
are arranged. This case is similar to the case of $d$-ary trees for which our
result in \cref{appendix:lower-bound} implies $\Omega\pr{\frac{n}{\log_d{n}}}$
lower bound on distinct node interventions.  However, if, for example, half of
the nodes form a line and the other half are all children of the last node on
the line as shown in \cref{fig:n-branch}, then the condition of
\cref{thm:log-condition} is still satisfied and \ac{raps} remains in the fast
regime in terms of the number of distinct node interventions.

\begin{figure}
\centering
\begin{tikzpicture}[->,every node/.style={node}]
\node (x1) {$X_1$};
\node[right=of x1] (x2) {$X_2$};
\node[right=of x2,draw=none] (dots1) {$\dots$};
\node[right=of dots1] (xno2) {$X_{n/2}$};
\node[below=of x1,yshift=1.5em] (p) {$P$};
\node[below=of p,yshift=1.5em,scale=0.75] (xno2p1) {$X_{n/2+1}$};
\node[right=of xno2p1,scale=0.75] (xno2p2) {$X_{n/2+2}$};
\node[right=of xno2p2,draw=none] (dots2) {$\dots$};
\node[right=of dots2,scale=0.9] (xnm1) {$X_{n-1}$};
\path
(p) edge (x1)
(p) edge (x2)
(p) edge (dots1)
(p) edge (xno2)
(p) edge (xno2p1)
(p) edge (xno2p2)
(p) edge (dots2)
(p) edge (xnm1);
\path
(x1) edge (x2)
(x2) edge (dots1)
(dots1) edge (xno2)
(xno2) edge (xno2p1)
(xno2) edge (xno2p2)
(xno2) edge (dots2)
(xno2) edge (xnm1);
\end{tikzpicture}
\caption{An example of a graph where  $P$ has $n$ children, but only half of
them form a line graph ($X_1\to\dots\to X_{n/2}$).  $\cO(\log^k n)$ expected
number of interventions still suffices because the other half of the nodes are
all children of $X_{n/2}$.}\label{fig:n-branch}\vspace{-.5cm}
\end{figure}
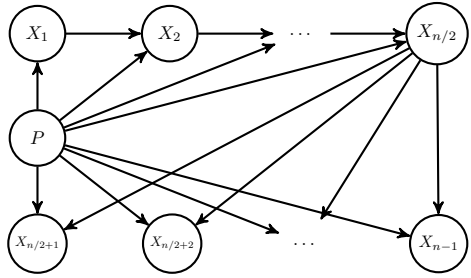

\subsubsection{Erd\H{o}s-R\'enyi Random Graphs}\label{sec:erdos-renyi}
In addition to providing examples of instances for which the condition of
\cref{thm:log-condition} holds, we show that Erd\H{o}s-R\'enyi random \acp{dag}
with large enough edge probability $p$ also satisfy this condition.  While
originally, Erd\H{o}s-R\'enyi model was proposed for undirected graphs, it
naturally extends to \acp{dag} as well \cite{hu2014randomized}.  For this,
label the nodes in $\cV$ from $1$ to $n$.  Next, select a permutation $\pi$
over $[n]$ uniformly at random.  Subsequently, for two nodes $i,j\in[n]$ such
that $i<j$ draw an edge $\cbr{i,j}$ with probability $p$. Finally, if an edge
$\cbr{i,j}$ is picked, orient it as $i\to j$ if $\pi(i)<\pi(j)$ and
$i\leftarrow j$ otherwise.  For such randomly generated \ac{dag}, the following
result, proved in \cref{appendix:erdos-renyi}, holds.

% Secondly, we construct an undirected Erd\H{o}s-R\'enyi random graph
% over $\cV$ that is each edge is included in the graph with probability
% $p$, independently from every other edge. Lastly, we orient the
% constructed graph using the selected permutation $\pi$. More
% precisely, an edge between two nodes $u$ and $v$ with labels $i$
% and $j$, respectively, will be oriented as $u\to v$ if $\pi(i)<\pi(j)$
% and $v\to u$, otherwise.  For such randomly generated \ac{dag}, we
% state the following result.
% proved in \cref{appendix:erdos-renyi}.
%Next, for two nodes $u,v\in\cV$ with labels $i$ and $j$, respectively, we draw
%an edge between $u$ and $v$,$\cbr{u,v}$ with probability $p$. Lastly, an edge
%$\cbr{i,j}$ is  oriented as $i\to j$ if $\pi(i)<\pi(j)$ and $i\leftarrow j$
%otherwise.

\begin{restatable}{corollary}{erdosrenyi}\label{corollary:erdos-renyi}
The family of Erd\H{o}s-R\'enyi random \acp{dag} satisfies the
condition of \cref{thm:log-condition} in expectation if
$p\geq1-\pr{\frac{1-c}{\log^kn-1}}^{1/(\log^kn-1)}$, for any
constant $c\in[0,1]$. Therefore, for such graphs, \ac{raps} requires
%\begin{equation*}
$\E[N(\bcG,P)]=\cO(\log^k n)$
%\end{equation*}
expected number of interventions.
\end{restatable}
As shown in \cref{remark:erdos-renyi} in \cref{appendix:erdos-renyi},
the result of \cref{corollary:erdos-renyi} holds for
$p\geq\frac{\log(\log^k(n)-1)-\log(1-c)}{\log^k(n)-1}$ and
asymptotically the lower bound for $p$ in \cref{corollary:erdos-renyi}
behaves as $\Theta\pr{\frac{k\log\log(n)}{\log^k(n)}}$.

\subsection{Slow Regime}
%The condition of \cref{thm:log-condition} is fairly restrictive thus,
In this section, we provide a bound on the expected number of interventions of
\ac{raps} under a more relaxed assumption than the condition in
\cref{thm:log-condition}.  The following \lcnamecref{thm:sublinear} states that
if there are at most $\cO\pr{\frac{n}{\log_d(n)}}$ nodes
$X\in\cV\setminus\cbr{P}$ such that all paths between $X$ and $P$ are inactive,
then the expected number of interventions required by \ac{raps} is bounded by
$\cO\pr{\frac{n}{\log_d(n)}}$, where $d$ is the maximum degree in the skeleton
of $\bcG$.  Under the faithfulness assumption \cite{pearl2009causality}, the
aforementioned condition means that there are at most
$\cO\pr{\frac{n}{\log_d(n)}}$ nodes in $\cV$ which are independent with the
parent node $P$.
\begin{restatable}{theorem}{sublinear}\label{thm:sublinear}
Let $\bcG$ be an arbitrary \ac{dag} in which there are at most
$\cO\pr{\frac{n}{\log_d(n)}}$ nodes $X\in\cV\setminus\cbr{P}$ such
that either $P$ is disconnected with $X$, or all paths between
$P$ and $X$ are blocked by colliders.  Then,
$
\E[N(\bcG,P)]=\cO\pr{\frac{n}{\log_d n}},
$
where $d$ is the maximum degree in the skeleton of $\bcG$.
\end{restatable}
The proof of \cref{thm:sublinear} is in \cref{appendix:sublinear} and in
\cref{appendix:lower-bound} we show that $d$-ary directed trees are worst case examples
for which the bound is tight even though such trees do not have colliders.

\section{Generalization to Multiple Parent Nodes}\label{sec:multiparent}

Let $\cP$ be the set of all parent nodes of the reward node. We generalize
\cref{alg:thealg} to an algorithm that finds all the parent nodes of the reward
node by repeatedly discovering each of the parent nodes in
\cref{alg:multiparent} with a more detailed version of the same algorithm
presented in \cref{alg:full} in the appendix.\hfill

\begin{wrapfigure}{r}{0.5\linewidth}
\begin{minipage}{\linewidth}
\begin{algorithm}[H]
\caption{\ac{raps} for discovering multiple parents}\label{alg:multiparent}
\begin{algorithmic}[1]
\State $\hat{\cP}\gets\emptyset, \cS\gets\cV$
\While{True}
\State$\hat{P}\gets$ the result of running \cref{alg:thealg}
providing it with $\cS$ and
$\hat{\cP}$\label{line:call-thealg}
\Comment See remark
\If{$\hat{P}=\varnothing$}
\State\Break
\EndIf
\State $\hat{\cP}\gets\hat{\cP}\cup\cbr{\hat{P}}$
\State $\cS\gets\cS\setminus\cD(\hat{P})$
\EndWhile
\Return $\hat{\cP}$
\end{algorithmic}
\end{algorithm}
\end{minipage}
\end{wrapfigure}
\paragraph{Remark.} On \cref{line:call-thealg} \cref{alg:multiparent} calls
\cref{alg:thealg} to find a next parent node $P$ with the starting candidate
set being equal to $\cS$.  While previously \cref{alg:thealg} could use only
atomic interventions to determine if an arbitrary $X\in\cV$ is an ancestor of
$P$, in this case this algorithm will intervene on $\hat{\cP}\cup\cbr{X}$ to
find if there exists a realization that changes $R$ by changing the value of
$X$ while keeping the other values in the intervention set constant. In other
words, the mean reward estimate and empirical distributions over all the nodes
under intervention will be compared to the corresponding values under all
interventions over $\hat{\cP}$ and the descendants of $X$ are determined
similarly. If a change under same values of $\hat\cP$ but different values of
$X$ occurs, then $X$ concluded to be an ancestor of some parent node in
$\cP\setminus\hat\cP$. \Cref{alg:multiparent} uses the observation that if
$P,P'\in\cP$ and $P\in\bar\cA(P')$, then $P'$ will be discovered by
\cref{alg:thealg}, but not $P$.  This happens because even if $P$ is intervened
on, the algorithm would have to exclude the descendants of $P$ before returning
$P$ as the parent of the reward node. Afterwards, by recalling
\cref{alg:thealg} and providing it with $\hat{\cP}$ that contains $P'$, it will
be able to discover $P$ as another parent node.

As an example, consider the graph in \cref{fig:multiparent-example}. During the
first call to \cref{alg:thealg} the node $P_2$ will be discovered. In the
second call to \cref{alg:thealg} with $\hat{\cP}=\{P_2\}$, there needs to be an
intervention on $P_2$ in order to cut the causal link from $P_1$ to $P_2$ to
determine whether $P_1$ is a parent of the reward node $R$.

\begin{wrapfigure}{r}{0.3\linewidth}
\centering
\begin{tikzpicture}[->,every node/.style=node]
\node (p1) {$P_1$};
\node[right=of p1] (p2) {$P_2$};
\node[right=of p2,style=dashed] (r) {$R$};
\path
(p1) edge[bend left] (r)
(p1) edge (p2)
(p2) edge (r);
\end{tikzpicture}
\caption{An example \ac{dag} with multiple parent nodes.}
\label{fig:multiparent-example}
\end{wrapfigure}
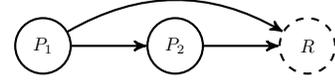
The following \lcnamecref{thm:multiparent} proved in
\cref{appendix:multiparent} generalizes the conditions of
\cref{thm:log-condition,thm:sublinear} such that the \cref{alg:multiparent}
discovers each parent node with the number of interventions as in
\cref{thm:log-condition,thm:sublinear}, respectively.  This is done by
considering all graphs from which the descendants of some subsequence of parent
nodes were removed. The nodes in the subsequence are selected as the last nodes
in some topological ordering of the nodes in $\cP$.
\begin{restatable}{theorem}{multiparent}\label{thm:multiparent}
Let $\btau(\cP)$ be the set of all topological orderings of the
parent nodes $\cP$. Assume that the condition of \cref{thm:log-condition}
holds for all graph-parent-node pairs with at least $c\log^k(n)$
nodes in the graph in the set
\begin{align*}\label{eq:multiparent-graphs}
\cbr{(\bcG_{\cV\setminus\cD(\cP)},\varnothing)}\cup
\biggl\{(\bcG_{\cV\setminus\cS(\tau,i)},\tau_i)
\bigm|\tau\in\btau(\cP), i\in[\abs{\cP}],
\cS(\tau,i)=\bigcup_{P\in\tau[i+1:]}\cD(P)\biggr\},
\end{align*}
where $\tau[i+1\hspace{-2pt}:]$ consists of the last $\abs{\cP}-i$ elements
of $\tau$, $\tau_i$ is the $i$-th element of $\tau$ and $c>0$ is some
constant.  Then the expected number of interventions required by
\cref{alg:multiparent} to find all parent nodes is
$\cO\pr{\abs{\cP}\log^k{n}}$.  Similarly, assume that all graph-parent-node
pairs in the set above with graphs of size at least $\frac{cn}{\log_d(n)}$
satisfy the condition of \cref{thm:sublinear}.  Then the expected number of
interventions required by \cref{alg:multiparent} is $\cO\pr{\frac{\abs{\cP}
n}{\log_d{n}}}$.
\end{restatable}

\Cref{thm:multiparent} gives general conditions for the upper bounds
on the number of interventions. Below, we combine our result for
Erd\H{o}s-R\'enyi graphs from \cref{sec:erdos-renyi} with the result
of \cref{thm:multiparent} to arrive at a condition on the probability
$p$ such that \cref{alg:multiparent} discovers all parents of the reward
node. The proof is in \cref{appendix:multiparent}.
\begin{restatable}{corollary}{ermultiparent}\label{corollary:ermultiparent}
Let $\bcG_{n,p}$ be an Erd\H{o}s-R\'enyi graph with
$p\geq1-\pr{\frac{1-c_0}{\log^k(c_1\log^k(n))-1}}^{
1/(\log^k(c_1\log^k(n))-1)}$ for some constants $c_0\in[0,1]$ and
$c_1\in\R_{>0}$, then to discover $\cP$,
\cref{alg:multiparent} needs $\E[N(\bcG,\cP)]=\cO\pr{\abs{\cP}\log^k(n)}$
expected number of interventions.
\end{restatable}

\vspace{-.5em}
\section{Experiments}

\begin{wrapfigure}{r}{0.3\linewidth}
\includegraphics[width=0.3\textwidth]{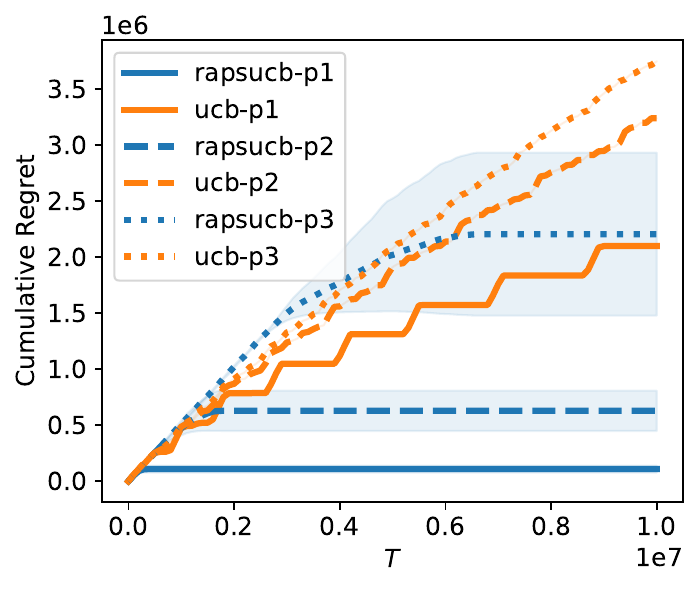}
\caption{Regret of \acs{ucb} and \ac{raps}+\acs{ucb} on
binary tree graph with $n=20$ nodes. The last character in the label
specifies the number of parent nodes.}\label{fig:regret}
\end{wrapfigure}

In this section we describe an experiment aimed at showing improved regret with
our approach. Other experiments that verify our theoretical findings about the
number of interventions that \ac{raps} takes, an analysis of the values of
$\eps$ and $\Delta$ in Erd\H{o}s-R\'enyi graphs, and the results of running the
combination of \ac{raps} and \acs{ucb} on such graphs are presented in
\cref{appendix:experiments}\footnote{The code to reproduce all of the
experiments is available at \url{https://github.com/BorealisAI/raps}.}. Here we
consider the case when the underlying graph is a binary tree with $n=20$ nodes
including the reward node, with the parent(s) of the reward node chosen
uniformly at random. The \ac{pcm} is such that the value of the root node is
chosen randomly and each other node passes the value of its' parent to its'
children with probability 0.9 and samples a value uniformly at random with
probability 0.1. All nodes in the graph take on binary values. The probability
of the reward node taking value $1$ is equal to the mean value of its' parents
and we set $\delta=0.01$. Even though our algorithm operates in the slow regime
on trees as discussed in \cref{appendix:lower-bound}, we still expect to see an
improvement since discovering the set of parent nodes drastically reduces the
set of arms that a standard multi-armed bandit algorithm needs to choose from.
The results are presented in \cref{fig:regret} with the number of parents
ranging from 1 to 3.  For each experiment we performed 10 independent runs and
presented their means and standard deviations. The results for \acs{ucb} do not
change significantly and thus the error bars for this algorithm could not be
seen. All our experiments (including the ones in the appendix) could be run
during a day on a single laptop.

\vspace{-.5em}
\section{Conclusion}
We proposed a causal bandit algorithm that does not require the knowledge of
the graph structure of the causal graph including the knowledge of the
essential graph. Our algorithm achieves improved regret by finding the set of
parent nodes of the reward node using causal structure of the arms and thus
reducing the number of arms a standard \ac{mab} algorithm needs to explore.

\nocite{*}
\bibliography{ref}

\begin{thebibliography}{28}
\providecommand{\natexlab}[1]{#1}
\providecommand{\url}[1]{\texttt{#1}}
\expandafter\ifx\csname urlstyle\endcsname\relax
  \providecommand{\doi}[1]{doi: #1}\else
  \providecommand{\doi}{doi: \begingroup \urlstyle{rm}\Url}\fi

\bibitem[Abbasi-Yadkori et~al.(2011)Abbasi-Yadkori, P{\'a}l, and
  Szepesv{\'a}ri]{abbasi2011improved}
Yasin Abbasi-Yadkori, D{\'a}vid P{\'a}l, and Csaba Szepesv{\'a}ri.
\newblock Improved algorithms for linear stochastic bandits.
\newblock \emph{Advances in neural information processing systems}, 24, 2011.

\bibitem[Agrawal and Goyal(2013)]{agrawal2013thompson}
Shipra Agrawal and Navin Goyal.
\newblock Thompson sampling for contextual bandits with linear payoffs.
\newblock In \emph{International conference on machine learning}, pages
  127--135. PMLR, 2013.

\bibitem[Capp{\'e} et~al.(2013)Capp{\'e}, Garivier, Maillard, Munos, and
  Stoltz]{cappe2013kullback}
Olivier Capp{\'e}, Aur{\'e}lien Garivier, Odalric-Ambrym Maillard, R{\'e}mi
  Munos, and Gilles Stoltz.
\newblock Kullback-leibler upper confidence bounds for optimal sequential
  allocation.
\newblock \emph{The Annals of Statistics}, pages 1516--1541, 2013.

\bibitem[De~Kroon et~al.(2022)De~Kroon, Mooij, and Belgrave]{de2022causal}
Arnoud De~Kroon, Joris Mooij, and Danielle Belgrave.
\newblock Causal bandits without prior knowledge using separating sets.
\newblock In \emph{Conference on Causal Learning and Reasoning}, pages
  407--427. PMLR, 2022.

\bibitem[Greenewald et~al.(2019)Greenewald, Katz, Shanmugam, Magliacane,
  Kocaoglu, Boix~Adsera, and Bresler]{greenewald2019sample}
Kristjan Greenewald, Dmitriy Katz, Karthikeyan Shanmugam, Sara Magliacane,
  Murat Kocaoglu, Enric Boix~Adsera, and Guy Bresler.
\newblock Sample efficient active learning of causal trees.
\newblock \emph{Advances in Neural Information Processing Systems}, 32, 2019.

\bibitem[Harris et~al.(2020)Harris, Millman, van~der Walt, Gommers, Virtanen,
  Cournapeau, Wieser, Taylor, Berg, Smith, Kern, Picus, Hoyer, van Kerkwijk,
  Brett, Haldane, del R{\'{i}}o, Wiebe, Peterson, G{\'{e}}rard-Marchant,
  Sheppard, Reddy, Weckesser, Abbasi, Gohlke, and Oliphant]{harris2020array}
Charles~R. Harris, K.~Jarrod Millman, St{\'{e}}fan~J. van~der Walt, Ralf
  Gommers, Pauli Virtanen, David Cournapeau, Eric Wieser, Julian Taylor,
  Sebastian Berg, Nathaniel~J. Smith, Robert Kern, Matti Picus, Stephan Hoyer,
  Marten~H. van Kerkwijk, Matthew Brett, Allan Haldane, Jaime~Fern{\'{a}}ndez
  del R{\'{i}}o, Mark Wiebe, Pearu Peterson, Pierre G{\'{e}}rard-Marchant,
  Kevin Sheppard, Tyler Reddy, Warren Weckesser, Hameer Abbasi, Christoph
  Gohlke, and Travis~E. Oliphant.
\newblock Array programming with {NumPy}.
\newblock \emph{Nature}, 585\penalty0 (7825):\penalty0 357--362, September
  2020.
\newblock \doi{10.1038/s41586-020-2649-2}.
\newblock URL \url{https://doi.org/10.1038/s41586-020-2649-2}.

\bibitem[Hauser and B{\"u}hlmann(2014)]{hauser2014two}
Alain Hauser and Peter B{\"u}hlmann.
\newblock Two optimal strategies for active learning of causal models from
  interventional data.
\newblock \emph{International Journal of Approximate Reasoning}, 55\penalty0
  (4):\penalty0 926--939, 2014.

\bibitem[Heckel et~al.(2019)Heckel, Shah, Ramchandran, and
  Wainwright]{heckel2019active}
Reinhard Heckel, Nihar~B Shah, Kannan Ramchandran, and Martin~J Wainwright.
\newblock Active ranking from pairwise comparisons and when parametric
  assumptions do not help.
\newblock \emph{The Annals of Statistics}, 47\penalty0 (6):\penalty0
  3099--3126, 2019.

\bibitem[Hu et~al.(2014)Hu, Li, and Vetta]{hu2014randomized}
Huining Hu, Zhentao Li, and Adrian~R Vetta.
\newblock Randomized experimental design for causal graph discovery.
\newblock \emph{Advances in neural information processing systems}, 27, 2014.

\bibitem[Hunter(2007)]{Hunter:2007}
J.~D. Hunter.
\newblock Matplotlib: A 2d graphics environment.
\newblock \emph{Computing in Science \& Engineering}, 9\penalty0 (3):\penalty0
  90--95, 2007.
\newblock \doi{10.1109/MCSE.2007.55}.

\bibitem[Huo and Fu(2017)]{huo2017risk}
Xiaoguang Huo and Feng Fu.
\newblock Risk-aware multi-armed bandit problem with application to portfolio
  selection.
\newblock \emph{Royal Society open science}, 4\penalty0 (11):\penalty0 171377,
  2017.

\bibitem[Jones(1994)]{jones1994generalized}
Charles~H Jones.
\newblock Generalized hockey stick identities and iv-dimensional blockwalking.
\newblock 1994.

\bibitem[Lattimore et~al.(2016)Lattimore, Lattimore, and
  Reid]{lattimore2016causal}
Finnian Lattimore, Tor Lattimore, and Mark~D Reid.
\newblock Causal bandits: Learning good interventions via causal inference.
\newblock \emph{Advances in Neural Information Processing Systems}, 29, 2016.

\bibitem[Lattimore and Szepesv{\'a}ri(2020)]{lattimore2020bandit}
Tor Lattimore and Csaba Szepesv{\'a}ri.
\newblock \emph{Bandit algorithms}.
\newblock Cambridge University Press, 2020.

\bibitem[Lee and Bareinboim(2018)]{lee2018structural}
Sanghack Lee and Elias Bareinboim.
\newblock Structural causal bandits: where to intervene?
\newblock \emph{Advances in Neural Information Processing Systems}, 31, 2018.

\bibitem[Liu et~al.(2020)Liu, See, Ngiam, Celi, Sun, Feng,
  et~al.]{liu2020reinforcement}
Siqi Liu, Kay~Choong See, Kee~Yuan Ngiam, Leo~Anthony Celi, Xingzhi Sun,
  Mengling Feng, et~al.
\newblock Reinforcement learning for clinical decision support in critical
  care: comprehensive review.
\newblock \emph{Journal of medical Internet research}, 22\penalty0
  (7):\penalty0 e18477, 2020.

\bibitem[Lu et~al.(2020)Lu, Meisami, Tewari, and Yan]{lu2020regret}
Yangyi Lu, Amirhossein Meisami, Ambuj Tewari, and William Yan.
\newblock Regret analysis of bandit problems with causal background knowledge.
\newblock In \emph{Conference on Uncertainty in Artificial Intelligence}, pages
  141--150. PMLR, 2020.

\bibitem[Lu et~al.(2021)Lu, Meisami, and Tewari]{lu2021causal}
Yangyi Lu, Amirhossein Meisami, and Ambuj Tewari.
\newblock Causal bandits with unknown graph structure.
\newblock \emph{Advances in Neural Information Processing Systems},
  34:\penalty0 24817--24828, 2021.

\bibitem[Mokhtarian et~al.(2022)Mokhtarian, Akbari, Jamshidi, Etesami, and
  Kiyavash]{mokhtarian2022learning}
Ehsan Mokhtarian, Sina Akbari, Fateme Jamshidi, Jalal Etesami, and Negar
  Kiyavash.
\newblock Learning bayesian networks in the presence of structural side
  information.
\newblock In \emph{Proceedings of the AAAI Conference on Artificial
  Intelligence}, volume~36, pages 7814--7822, 2022.

\bibitem[Nair et~al.(2021)Nair, Patil, and Sinha]{nair2021budgeted}
Vineet Nair, Vishakha Patil, and Gaurav Sinha.
\newblock Budgeted and non-budgeted causal bandits.
\newblock In \emph{International Conference on Artificial Intelligence and
  Statistics}, pages 2017--2025. PMLR, 2021.

\bibitem[Pearl(2009)]{pearl2009causality}
Judea Pearl.
\newblock \emph{Causality}.
\newblock Cambridge university press, 2009.

\bibitem[Sawant et~al.(2018)Sawant, Namballa, Sadagopan, and
  Nassif]{sawant2018contextual}
Neela Sawant, Chitti~Babu Namballa, Narayanan Sadagopan, and Houssam Nassif.
\newblock Contextual multi-armed bandits for causal marketing.
\newblock \emph{arXiv preprint arXiv:1810.01859}, 2018.

\bibitem[Sen et~al.(2017)Sen, Shanmugam, Dimakis, and
  Shakkottai]{sen2017identifying}
Rajat Sen, Karthikeyan Shanmugam, Alexandros~G Dimakis, and Sanjay Shakkottai.
\newblock Identifying best interventions through online importance sampling.
\newblock In \emph{International Conference on Machine Learning}, pages
  3057--3066. PMLR, 2017.

\bibitem[Shanmugam et~al.(2015)Shanmugam, Kocaoglu, Dimakis, and
  Vishwanath]{shanmugam2015learning}
Karthikeyan Shanmugam, Murat Kocaoglu, Alexandros~G Dimakis, and Sriram
  Vishwanath.
\newblock Learning causal graphs with small interventions.
\newblock \emph{Advances in Neural Information Processing Systems}, 28, 2015.

\bibitem[Silva et~al.(2022)Silva, Werneck, Silva, Pereira, and
  Rocha]{silva2022multi}
N{\'\i}collas Silva, Heitor Werneck, Thiago Silva, Adriano~CM Pereira, and
  Leonardo Rocha.
\newblock Multi-armed bandits in recommendation systems: A survey of the
  state-of-the-art and future directions.
\newblock \emph{Expert Systems with Applications}, 197:\penalty0 116669, 2022.

\bibitem[Squires et~al.(2020)Squires, Magliacane, Greenewald, Katz, Kocaoglu,
  and Shanmugam]{squires2020active}
Chandler Squires, Sara Magliacane, Kristjan Greenewald, Dmitriy Katz, Murat
  Kocaoglu, and Karthikeyan Shanmugam.
\newblock Active structure learning of causal dags via directed clique trees.
\newblock \emph{Advances in Neural Information Processing Systems},
  33:\penalty0 21500--21511, 2020.

\bibitem[Yao(1977)]{yao1977probabilistic}
Andrew Chi-Chin Yao.
\newblock Probabilistic computations: Toward a unified measure of complexity.
\newblock In \emph{18th Annual Symposium on Foundations of Computer Science
  (sfcs 1977)}, pages 222--227. IEEE Computer Society, 1977.

\bibitem[Zhang et~al.(2022)Zhang, Chen, and Singh]{zhang2022causal}
Jingwen Zhang, Yifang Chen, and Amandeep Singh.
\newblock Causal bandits: Online decision-making in endogenous settings.
\newblock \emph{arXiv preprint arXiv:2211.08649}, 2022.

\end{thebibliography}
\bibliographystyle{plainnat}

\newpage
\appendix
\onecolumn

\section{Exact Number of Interventions}\label{appendix:exact}
\begin{algorithm}[tb]
\caption{An algorithm equivalent to \cref{alg:thealg}.}\label{alg:equiv}
\begin{algorithmic}[1]
\Require{Set of nodes $\cV$ of $\bcG$ given as input}
\Ensure{The parent node $P\in\cV$ of the reward node or $\varnothing$
if there is no parent node in $\cV$}
\State Sample a random permutation $\tau$ of nodes in $\cV$
\State $\hat{P}\gets\varnothing,i\gets0$
\For{$X\in\tau$}
\State $i\gets i+1$
\If{$\cA(P)\triangle\cA(X)\cap(\tau_1,\dots,\tau_{i-1})\neq\emptyset$}
\label{line:non-common-parent-condition}
\State\Continue
\EndIf
\State Intervene on $X$ to determine if $P\in\cD(X)$
\If{$P\in\cD(X)$}
\State $\hat{P}\gets X$
\EndIf
\EndFor
\State\Return $\hat{P}$
\end{algorithmic}
\end{algorithm}

\exact*
\begin{proof}
First we show that \cref{alg:thealg} is equivalent to \cref{alg:equiv} in the
sense that the same sequences of nodes are intervened on by both algorithms
with the same probability. Although the \cref{alg:equiv} is not practical
because it uses the graph structure on \cref{line:non-common-parent-condition}, it allows us to present a proof for this theorem.
This algorithm first samples a permutation $\tau$ of nodes in $\cV$ and then
intervenes on a node $X\in\tau$ only if none of the nodes in
$\cA(P)\triangle\cA(X)$ appeared before $X$ in the permutation.
As an example, suppose that \cref{alg:equiv} selects permutation $(X_3,X_2,X_1,P)$  in \cref{fig:example}.
Then the nodes intervened on by \cref{alg:equiv} are the same as the nodes
intervened on by \cref{alg:thealg} in the example in \cref{sec:algorithm}. This
is because $X_2$ is a descendant of $X_3$ (and $X_3$ is not an ancestor of $P$)
and therefore it will not be intervened on.
Moreover, if \cref{alg:equiv} selects permutations
$(X_3,X_1,X_2,P)$ and $(X_3,X_1,P,X_2)$, the resulting intervention sequences would be  the same run of
\cref{alg:equiv}.

By induction on the intervened on nodes, the base case is clear since in both
\cref{alg:thealg} and \cref{alg:equiv} the first node is sampled with
probability $1/n$ and always intervened on. Let $\bW=(W_1,\dots,W_l)$ be a
uniformly random permutation of any $l$ elements of $\cV$ and $\bW'$ be a
subsequence of $\bW$ with an element of $W\in\bW$ included in $\bW'$ if no
element of $\cA(P)\triangle\cA(W)\setminus\cbr{W}$ was included before it. For
the example in \cref{fig:example} and sequence $\bW=(X_3,X_2)$ we have
$\bW'=(X_3)$ since, as mentioned before,
$X_3\in\cA(P)\triangle\cA(X_2)\setminus\cbr{X_2}$. Note that in
\cref{alg:thealg} a node could be intervened on only when it was not intervened
on before and none of the non-common ancestors of that node and the parent node
were intervened on. Thus, let
$\cS=\cbr{X\in\cV:\cA(P)\triangle\cA(X)\cup\cbr{X}\cap\bW'=\emptyset}$ be the
set of nodes that could be intervened on by \cref{alg:thealg} in the next
round. The probability that a node $X\in\cS$ is sampled by \cref{alg:thealg}
given the sequence $\bW'$ is $1/s$, where $s=\abs\cS$.
On the other hand, the probability that a node
$X\in\cS$ is intervened on next by \cref{alg:equiv} is
\begin{align}
&\prob\cbr{\text{$X$ is sampled before
$\cA(P)\triangle\cA(X)\setminus\cbr{X}$}|\bW} \\
&\quad=\sum_{k=0}^{n-l-s}\prob\cbr{
\bW_{l+1:l+k}\cap\pr{\cA(P)\triangle\cA(X)\setminus\cbr{X}}=\emptyset
\text{ and }W_{l+1+k}=X|\bW} \\
&\quad=\sum_{k=0}^{n-l-s}\frac{\binom{n-l-s}{k}k!(n-l-k-1)!}{(n-l)!} \\
&\quad=\sum_{k=0}^{n-l-s}\frac{(n-l-s)!(n-l-k-1)!}{(n-l-s-k)!(n-l)!} \\
&\quad=\frac1s\sum_{k=0}^{n-l-s}\frac{\binom{n-l-k-1}{s-1}}{\binom{n-l}{s}}
=\frac1s,
\end{align}
where $\bW_{l+1:l+k}$ consists of $k$ elements sampled uniformly at random
after sampling $\bW$, we used the fact that there are $n-l-s$ ``good`` elements
from which we need to sample $k$ elements before sampling $W_{l+1+k}=X$ while
the remaining elements could come in any order, and to get the last equality we
used the hockey-stick identity~\cite{jones1994generalized}. The hockey-stick
identity states that for any $n,r\in\mathbb{N}$ such that $n\geq r$ it holds
that $\sum_{i=r}^n\binom{i}{r}=\binom{n+1}{r+1}$. By the chain rule of
probability, the intermediate result holds.

Next, with $\bW=(W_1,\dots,W_n)$ being a uniformly random permutation
corresponding to a run of \cref{alg:equiv}, and defining
$A_X=\cbr{\text{\cref{alg:equiv} intervenes on $X$}}$,
$\bW_{<i}=(W_1,\dots,W_{i-1})$ we can get
\begin{align}
&\E[N(\bcG,P)]=\E[\sum_{X\in\cV}\I\cbr{A_X}]=\sum_{X\in\cV}\prob\cbr{A_X} \\
&\quad=\sum_{X\in\cV}\sum_{i=1}^{n-\abs{\cA(P)\triangle\cA(X)\setminus\cbr{X}}}
\prob\cbr{\cA(P)\triangle\cA(X)\setminus\cbr{X}\cap\bW_{<i}=\emptyset|W_i=X}
\prob\cbr{W_i=X} \\
&\quad=\sum_{X\in\cV}\sum_{i=1}^{n-\abs{\cA(P)\triangle\cA(X)\setminus\cbr{X}}}
\frac{
\binom{n-\abs{\cA(P)\triangle\cA(X)\setminus\cbr{X}}-1}{i-1}(i-1)!(n-i)!}{
(n-1)!}
\cdot\frac1n \\
&\quad=\sum_{X\in\cV}\frac1{\abs{\cA(P)\triangle\cA(X)\setminus\cbr{X}}+1}
\sum_{i=1}^{n-\abs{\cA(P)\triangle\cA(X)\setminus\cbr{X}}}
\frac{\binom{n-i}{\abs{\cA(P)\triangle\cA(X)\setminus\cbr{X}}}}{
\binom{n}{\abs{\cA(P)\triangle\cA(X)\setminus\cbr{X}}+1}} \\
&\quad=\sum_{X\in\cV}\frac1{\abs{\cA(P)\triangle\cA(X)\setminus\cbr{X}}+1},
\end{align}
where we used the fact that to have
$\cA(P)\triangle\cA(X)\setminus\cbr{X}\cap\bW_{<i}=\emptyset$ we need to sample
$i-1$ ``good`` elements from $n-\abs{\cA(P)\triangle\cA(X)\setminus\cbr{X}}-1$
elements (since $W_i$ is fixed to be $X$) while the rest of the $n-i$ elements
could be shuffled, and to get the last equality we used the hockey-stick
identity \cite{jones1994generalized}.
\end{proof}

\section{Logarithmic Upper Bound}\label{appendix:log-bound}

\cfamily*

The proof of this lemma uses the following proposition.
\begin{proposition}\label{proposition:descendant-subgraph}
For a graph $\bcG=(\cV,\cE)$ let $\cU\subseteq\cV$ and $\cS\supseteq\cD(\cU)$,
then $\cD_\cS(\cU)=\cD(\cU)$.
\end{proposition}
\begin{proof}
Let $X\in\cD(\cU)$, then there is a directed path from some $U\in\cU$
to $X$ in $\bcG$. Each node on this path belongs to
$\cD(U)\subseteq\cD(\cU)\subseteq\cS$.  Thus, in the graph $\bcG_\cS$
this path is preserved and we get $X\in\cD_\cS(\cU)$.  Conversely,
if $X\in\cD_\cS(\cU)$, then there is a directed path from some $U\in\cU$
to $X$ in $\bcG_\cS$. Adding vertices and connections to a graph
does not remove existing paths and thus there is a path from $U\in\cU$
to $X$ in $\bcG$ and therefore $X\in\cD(\cU)$.
\end{proof}
\begin{proof}[Proof of lemma \cref{lemma:candidate-family}]
In the proof we omit writing $\bcC_{\bcG(P)}$ and write $\bcC$
instead for simplicity. The proof is by induction. For the base
case we have $\cC=\cV=\cV\setminus\cD(\emptyset)\in\bcC$. By induction
hypothesis
\begin{align}
\cC&=\cV\setminus\cD(\cW) \text{ for some }
\cW\subseteq \cV\setminus\cA(P) \text{, or}
\label{eq:cc-first}\\
\cC&=\bar\cD(Z)\setminus\cD_{\cD(Z)}(\cW)\text{ for some }
Z\in\cA(P),\cW\subseteq\cD(Z)\setminus\cA_{\cD(Z)}(P).
\label{eq:cc-second}
\end{align}
Notice that in \cref{eq:cc-second} we have
$\cD(\cW)\subseteq\cD(\cD(Z)\setminus\cA_{\cD(Z)}(P))\subseteq\cD(Z)$ and
thus by \cref{proposition:descendant-subgraph} it could be rewritten
as
\begin{align}
C=\bar\cD(Z)\setminus\cD(\cW)\text{ for some } Z\in\cA(P),
\cW\subseteq\cD(Z)\setminus\cA_{\cD(Z)}(P).\label{eq:cc-third}
\end{align}
Next, we will consider four cases depending on whether the condition
$P\in\cD_\cC(X)$ on \cref{line:condition} of \cref{alg:thealg}
holds and whether $\cC$ conforms to \cref{eq:cc-first} or
\cref{eq:cc-third}.
\begin{enumerate}[label=Case \Roman*:,wide=0pt]
\item $P\in\cD_\cC(X)$. Consider the set $\cD(X)\setminus\cD(\cW')$
for some $\cW'\subseteq\cD(X)$. It consists precisely of descendants
of $X$ but not $Y\in\cW'\subseteq\cD(X)$.  Thus, we can remove all
nodes $Y\in\cD(\cW')$ from the original graph $\bcG$ and the
descendants of $X$ in this new graph will be equal to
$\cD(X)\setminus\cD(\cW')$:
\begin{align}
\cD(X)\setminus\cD(\cW')=\cD_{\cV\setminus\cD(\cW')}(X).
\end{align}
Next, we consider two subcases depending on whether $\cC$ conforms
to \cref{eq:cc-first} or \cref{eq:cc-third}.
\begin{enumerate}[label=\roman*),wide=0pt]
\item\label{item:candidate-family-subcase-cc-first} $\cC$ is such
that \cref{eq:cc-first} holds for some $\cW\subseteq\cV\setminus\cA(P)$.
Then using the above, it is left to show that there exists
$\cW'\subseteq\cD(X)\setminus\cA_{\cD(X)}(P)$ such that
$\cD_{\cV\setminus\cD(\cW')}(X)=\cD_{\cV\setminus\cD(\cW)}(X)
=\cD_\cC(X)$ since $\bar\cD(X)$ is what the recursive function
will be called with in \cref{alg:thealg}.  For this we can set
$\cW'=\cD(\cW)\cap\cD(X)$. Indeed,
$\cW$ does not contain any ancestors of $P$ and thus $\cD(\cW)$
does not contain ancestors of $P$ in $\bcG$ which means that
$\cW'\cap\cA_{\cD(X)}(P)=\emptyset$. Moreover, $\cV\setminus\cD(\cW)$
removes only non-descendants of $X$ from $\bcG$ compared to
$\cV\setminus\cD(\cW')$. This is true because
\begin{align}
\cD(X)\cap(\cV\setminus(\cD(\cW)\cap\cD(X)))
=\cD(X)\setminus\cD(\cW)\subseteq\cV\setminus\cD(\cW)
\end{align}
and $\bar\cD_\cC(X)\in\bcC$ follows from
\cref{proposition:descendant-subgraph}.

\item $\cC$ is such that \cref{eq:cc-third} holds for some
$Z\in\cA(P),\cW\subseteq\cD(Z)\setminus\cA_{\cD(Z)}(P)$. Similarly
to the \cref{item:candidate-family-subcase-cc-first} for
$\cW'=\cD(\cW)\cap\cD(X)$ we get
\begin{align}
\cD(X)\cap(\cV\setminus(\cD(\cW)\cap\cD(X)))=\cD(X)\setminus\cD(\cW)
\subseteq\bar\cD(Z)\setminus\cD(\cW),
\end{align}
where the last step follows from the fact that $Z\in\bar\cA(X)$
which is true because $X\in\cC\subseteq\bar\cD(Z)$. Additionally,
$\cW'\subseteq\cD(X)\setminus\cA_{\cD(X)}(P)$. This is true because
$\cW$ does not contain the ancestors of $P$ in $\bcG_{\cD(Z)}$ and $\cD(X)$
is a subset of $\cD(Z)$ since $Z\in\bar\cA(X)$.  Finally, the fact
that $\bar\cD_{\cC}(X)\in\bcC$ follows from combining the results
above and \cref{proposition:descendant-subgraph}.
\end{enumerate}

\item $P\in\cD_\cC^c(X)$ which is equivalent to $X\in\cA_\cC^c(P)$.
In fact, we will show that $X\not\in\cA(P)$ by contradiction.
To this end, assume that $X\in\cA(P)$. In \cref{alg:thealg}
the candidate set
$\cC$ with which the recursive function is called decreases in
size during each call. At the same time, if at some point sample
$Y\in\cA(P)\cap\cA^c(X)$ gets sampled, then the candidate set will
consists of a subset of descendants of $Y$ to which $X$ does not
belong. Thus, to sample $X$ at some point we must not have sampled
such $Y$ before which means $\cA(P)\cap\cA^c(X)\subseteq\cC$. In
particular, this means that $\cA(P)\cap\cD(X)\subseteq\cC$ which
intern leads to $X\in\cA_\cC^c(P)$ which is a contradiction. Thus,
$X\not\in\cA(P)$.

In what follows, we again consider two subscases depending on whether
$\cC$ conforms to \cref{eq:cc-first} or \cref{eq:cc-third}.

\begin{enumerate}[label=\roman*),wide=0pt]
\item $\cC=\cV\setminus\cD(\cW)$ for some $\cW\subseteq\cA^c(P)$.
First, we show that
\begin{align}
\cD(X)\setminus\cD_\cC(X)\subseteq\cD(\cW).
\end{align}
Indeed, for $Y\in\cD(X)\setminus\cD_\cC(X)$ there must be a directed
path from $X$ to $Y$ in $\bcG$ but not in $\bcG_{\cV\setminus\cD(\cW)}$
which corresponds to the original graph with the descendants of
$\cW$ removed. Therefore, the descendants of $\cW$ block all the
directed paths from $X$ to $Y$ and therefore $Y\in\cD(\cW)$.  Next,
notice that since $X\in\cA^c(P)$ we can set $\cW'=\cW\cup\cbr{X}
\subseteq\cA^c(P)$ and using the result above get
\begin{align}
\cC\setminus\cD_\cC(X)=\cV\setminus\cD(\cW)\setminus\cD_\cC(X)
=\cV\setminus\cD(\cW)\setminus\cD(X)=\cV\setminus\cD(\cW')\in\bcC.
\end{align}
\item $\cC=\bar\cD(Z)\setminus\cD(\cW)$ for some $Z\in\cA(P),
\cW\subseteq\cD(Z)\setminus\cA_{\cD(Z)}(P)$. Again, consider directed
paths from $X$ to any $Y\in\cD(X)\setminus\cD_\cC(X)$.
Note that
$\cD(X)\setminus\cD_\cC(X)=\cD_{\bar\cD(Z)}(X)\setminus\cD_{\bar\cD(Z)\setminus\cD(\cW)}(X)$ where
we used \cref{proposition:descendant-subgraph} with the fact that
$X\in\bar\cD(Z)$ which implies $\cD(X)\subseteq\bar\cD(Z)$, and the
definition of $\cC$. Similarly to the previous item we get that
removing the set $\cD(\cW)$ from $\bcG_{\bar\cD(Z)}$ removes all
paths from $X$ to $Y$ in this graph and thus $Y\in\cD(\cW)$ or
$\cD(X)\setminus\cD_\cC(X) \subseteq\cD(\cW)$. Setting
$\cW'=\cW\cup\cbr{X}$ as before gives
$\cW'\subseteq\cD(Z) \setminus\cA_{\cD(Z)}(P)$ since
$X\in\bar\cD(Z)$ and as stated above $X\not\in\cA(P)$.
Thus, by \cref{proposition:descendant-subgraph} we get the
desired result.\qedhere
\end{enumerate}
\end{enumerate}
\end{proof}

\logcond*
\begin{proof}
Based on the definition of \cref{alg:thealg}, it is straightforward to see that the following recursion holds
%The proposed algorithm satisfies the following recursion
\begin{equation}
T(\cC)=\frac1{\abs{\cC}}\sum_{X\in\cA_\cC(P)} T(\bar\cD_\cC(X))
+ \frac1{\abs\cC}\sum_{X\in\cA_\cC^c(P)} T(\cD_\cC^c(X))
+ 1,\label{eq:recursion}
\end{equation}
where $T(\cC)$ denotes the number of interventions performed by
the recursive function in \cref{alg:thealg} given candidate set $\cC$.
We will
show that $T(\cC)\leq c'\log\abs\cC + c\log^k(n)$ for some $c'>0$ by considering two cases: $\abs\cC\leq c\log(n)$ and $\abs\cC> c\log(n)$.
The case where $\abs\cC\leq c\log(n)$ is straightforward since each node in $\cC$ is intervened on at most once and $\abs\cC\leq c\log^k(n)$.

Next, consider the case $\abs\cC>c\log^k(n)$. We provide a proof for
the case when $\abs{\cH_\cC(\alpha)}\geq\beta\abs\cC$ then comment
on why the result holds when $\abs{\cA_\cC(P)}\geq\gamma \abs\cC$.
By \cref{lemma:candidate-family} we have that for all $X\in\cA_\cC(P)$
it holds that $\bar\cD_\cC(X)\in\bcC$ and for all $X\in\cA_\cC^c(P)$
it holds that $\cD_\cC^c(X)\in\bcC$, thus the condition of the
theorem holds for the recursive calls and we can use induction
hypothesis after which it is left to check that
\begin{align}
\frac{c'}{\abs\cC}
&\sum_{X:X\in\cA_\cC(P)\wedge\abs{\cD_\cC(X)}>1}\log(\abs{\cD_\cC(X)}-1)
+\frac{c'}{\abs\cC}\sum_{X\in\cA_\cC^c(P)}\log(\abs\cC-\abs{\cD_\cC(X)})
+ 1\label{eq:log-condition-induction-first} \\
&\quad + \frac{c}{\abs\cC}\sum_{X\in\cA_\cC(P)}\log^k(n)
+ \frac{c}{\abs\cC}\sum_{X\in\cA_\cC^c(P)}\log^k(n)
\label{eq:log-condition-induction-second}
\end{align}
is bounded above by $c'\log\abs\cC + c\log^k(n)$.  First, note that
the \cref{eq:log-condition-induction-second} is bounded by $c\log^k(n)$
since $\cA_\cC(X)\cup\cA_\cC^c(X)=\cC$. Next, consider the
\cref{eq:log-condition-induction-first}. Note that for
$X\not\in\cH_\cC(\alpha)$ we can upper bound $\abs{\cD_\cC(X)}-1$
and $\abs\cC-\abs{\cD_\cC(X)}$ by $\abs\cC$.  At the same time, for
$X\in\cH_\cC(\alpha)$ we have
$\abs\cC-\abs{\cD_\cC(X)}\leq(1-\alpha)\abs\cC$.  Then our goal is
to show
\begin{align}
&\frac{c'(\abs\cC-\abs{\cH_\cC(\alpha)})}{\abs\cC}\log\abs\cC
+ \frac{c'\abs{\cH_\cC(\alpha)}}{\abs\cC}\log((1-\alpha)\abs\cC) + 1\\
&\qquad\leq c'\log\abs\cC
= \frac{c'(\abs\cC-\abs{\cH_\cC(\alpha)})}{\abs\cC}\log\abs\cC
+ \frac{c'\abs{\cH_\cC(\alpha)}}{\abs{\cC}}\log\abs{\cC},
\end{align}
rearranging we get
\begin{align}
\frac{\abs\cC}{c'\abs{\cH_\cC(\alpha)}}
\leq\log\frac{1}{1-\alpha}
=\log\pr{\frac{\alpha}{1-\alpha}+1}.
\end{align}
Notice that since for any $x> -1$ it holds that
$\frac{x}{1+x}\leq\log(x+1)$, it suffices to show
\begin{align}
\frac{\abs\cC}{c'\abs{\cH_\cC(\alpha)}}
\leq\alpha \iff \abs\cC\leq c'\abs{\cH_\cC(\alpha)}\alpha.
\end{align}
From our assumption on $\abs{\cH_\cC(\alpha)}$ it suffices to pick
$c'\geq\frac1{\alpha\beta}$. For the base case consider $\cC=\cbr{X,Y}$
then for $X\neq P$ we have that either there is a single edge
$P\rightarrow X$ or $P$ is disconnected with $X$ for the condition
of the theorem to hold. Then the recursion in \cref{eq:recursion}
could be rewritten as
\begin{equation}
T(\cbr{X,Y}) = \frac12 T(\cbr{Y}) + \frac12 T(\cbr{X}) + 1,
\end{equation}
from which it follows that
\begin{equation}
T(\cbr{X,Y})=2\leq c\log(2).
\end{equation}

For the case when $\abs{\cA_\cC(P)}\geq\gamma\abs\cC$ consider the
set $\cH'$ of size at least $\frac\gamma2\abs\cC$ of the ancestors of
$P$ which are closest to $P$ in the topologically sorted order in
the graph $\bcG_\cC$. Each node $X\in\cH'$ has at most $\abs\cC-\frac\gamma2
\abs\cC$ descendants.  The rest of the proof follows similar steps as
for the case when $\abs{\cH_\cC(\alpha)}\geq\beta\abs\cC$.
\end{proof}

\section{Fast Parent Discovery in Erd\H{o}s-R\'enyi Graphs}
\label{appendix:erdos-renyi}

In this subsection we show that Erd\H{o}s-R\'enyi graphs satisfy
the condition for fast discovery of the parent node for sufficiently
large  values of $p$. We first state the following theorem.

\begin{theorem}\label{thm:erdos-renyi}
Let $\bcG_{n,p}$ be Erd\H{o}s-R\'enyi random DAG with probability of
each edge between $n$ nodes being equal to $p$. Assume
$p\geq1-\pr{\frac{1-c}{n-1}}^{1/(n-1)}$ for some constant $c\in[0,1]$,
and denote by $X, Y$ the first and last nodes in the topological
order, respectively.  It holds that
\begin{align}
\E\abs{\cD(X)}&\geq c n,\text{ and}\\
\E\abs{\cA(Y)}&\geq c n.
\end{align}
\end{theorem}
\begin{proof}
In the proof we show by induction that the expected number of
descendants of the root node (the first node in topological
order) is lower bounded by $c n$. The expected number of the ancestors
could be lower bounded using the same reasoning.  Denote by $p_{n,i}$
the probability that there are exactly $i$ descendants of the root
node in the graph $\bcG_{n,p}$ and note that
\begin{equation}
\E\abs{\cD(X)}=\sum_{i=1}^n i p_{n,i}.
\end{equation}
Furthermore, $p_{n,i}$ satisfies the following recursion:
\begin{align}
p_{n,i}&=(1-(1-p)^{i-1})p_{n-1,i-1} + (1-p)^i p_{n-1,i} \\
&=p_{n-1,i-1} + (1-p)^{i-1}((1-p)p_{n-1,i} - p_{n-1,i-1}),
\end{align}
with $p_{1,1} = 1$ and $p_{n,i}=0$ if $i>n$ or $i=0$. Thus, we can
write
\begin{align}
\E\abs{\cD(X)}&=\sum_{i=1}^n i p_{n-1,i-1} +
\sum_{i=1}^n i (1-p)^{i-1}((1-p) p_{n-1,i} - p_{n-1,i-1})\\
&=\sum_{i=1}^n i p_{n-1,i-1}
- \sum_{i=1}^n (1-p)^{i-1} p_{n-1,i-1} \\
&\quad+ \sum_{i=1}^n \br{i(1-p)^{i} p_{n-1,i}
-(i-1)(1-p)^{i-1}p_{n-1,i-1}}.
\end{align}
Note that
\begin{align}
\sum_{i=1}^n\br{i(1-p)^i p_{n-1,i} - (i-1)(1-p)^{i-1} p_{n-1,i-1}} = 0
\end{align}
as a telescoping sum and by induction hypothesis we have that
\begin{align}
\sum_{i=1}^n i p_{n-1,i-1}=\sum_{i=1}^n (i-1) p_{n-1,i-1}+\sum_{i=1}^n p_{n-1,i-1}\geq c(n-1)+1,
\end{align}
therefore it is left to prove
\begin{align}
\sum_{i=1}^n(1-p)^{i-1} p_{n-1,i-1} \leq 1-c.
\end{align}

To prove this we first show that $p_{n,i}\leq(1-p)^{n-i}$, again
by induction. This holds for $n=1$ and all $i$ or $i=0$ and all
$n>1$. Furthermore by induction hypothesis we have
\begin{align}
p_{n,i}&\leq (1-(1-p)^{i-1})(1-p)^{n-i} + (1-p)^i(1-p)^{n-1-i} \\
&=(1-p)^{n-i} + (1-p)^{n-1} + (1-p)^{n-1} = (1-p)^{n-i}.
\end{align}
Using this result together with the fact that $p_{n-1,0}=0$ we have
\begin{align}
\sum_{i=1}^n(1-p)^{i-1} p_{n-1,i-1}\leq (n-1)(1-p)^{n-1} \leq 1-c,
\end{align}
where the last inequality follows from the assumption of the theorem.
To finish the proof, note that for $n=1$ we have $\E\abs{\cD(X)}=1\geq c$.
\end{proof}

\erdosrenyi*
\begin{proof}
From our lower bound on $p$ and \cref{thm:erdos-renyi} it follows
that for all subgraphs of size at least $\log^k n$ the first and
last nodes in the topological order have at least $c n$ descendants
and ancestors respectively. Let $\cC$ be an arbitrary candidate set
of size larger than $4\log^k n$ considered by \cref{alg:thealg}
when run on the graph $\bcG_{n,p}$.  Let $j\in[m]$ with $m=\abs{\cC}$
be the index of $P$ in the topologically sorted order in the subgraph
of $\bcG_{m,p}$ over nodes in $\cC$. We will comment bellow on the
situation when $P\not\in\cC$. We consider two cases. First, if
$j\leq m/2$, then consider $m/4$ subgraphs each consisting of
$m-m/2-i$ last nodes for $i\in[{m/4}]$ of the original graph
$\bcG_n$. The size of each of these subgraph is at least $\log^k
n$ and by \cref{thm:erdos-renyi} we have that each node at index
${m/2}+i-1$ in the topological order of the original graph $\bcG_m$
has at least $c{m/4}$ descendants. Since there are ${m/4}$ such
nodes, the first condition of \cref{thm:log-condition} is
satisfied.  The same happens for the cases when $P\not\in\cC$ because
in that case all the nodes in $\cC$ are non-ancestors of $P$ since
for $P\not\in\cC$ it must be the case that the algorithm intervened
on $P$ at some point before considering $\cC$.  Second, if $j>{m/2}$,
then by \cref{thm:erdos-renyi} we have that the number of
ancestors of $P$ is at least $c{n/2}$ which means that the second
condition of \cref{thm:log-condition} is satisfied.
\end{proof}

\begin{remark}\label{remark:erdos-renyi}
Note that since $(1-1/n)^n\leq e^{-1}$ we have $\log(n/c)\log\pr{1-1/n}
\leq -\frac{\log(n/c)}{n}$ and thus
$\pr{1-1/n}^{\log(n/c)}\leq\pr{\frac{c}{n}}^{1/n}$. Using this
together with the Bernoulli inequality $(1+x)^r\geq 1+rx$ for
$x\geq-1$ and $r\geq1$ we get that assuming $\log^kn\geq1+\max(1,(1-c)e)$
the condition of \cref{corollary:erdos-renyi} is satisfied for
$p\geq\frac{\log\pr{\log^k(n)-1}-\log(1-c)}{\log^k(n) - 1}$.
Additionally, by using L'H\'opital's rule and the fact that
$\lim_{n\to\infty}\pr{1/n}^{1/n}=1$, we get that the two lower bounds for $p$ presented in \cref{corollary:erdos-renyi} and here are asymptotically equivalent.
\end{remark}

\section{Sublinear Upper Bound}\label{appendix:sublinear}

\sublinear*
\begin{proof}
Let $A_X=\cbr{\text{the algorithm intervenes on the node $X$}}$, then
\begin{equation}
\E[N(\bcG,P)]=\E[\sum_{X\in \cV}\I\cbr{A_X}]=\sum_{X\in\cV} \prob(A_X).
\end{equation}
We split the sum above into three parts.  First, consider the nodes
$X$ that are at distance at most $m$ from the parent node for some
$m$ to be specified later. There are at most $d^{m+1}$ such nodes
and for each of them we bound the probability $\prob(A_X)$ by 1.
Similarly, for all nodes $X$ such that there is no collider-free
path between $P$ and $X$ we also bound the probability $\prob(A_X)$
by 1.  Each of the leftover nodes has a collider-free path of length
at least $m$ to $P$. We will show that this means that the probability
$\prob(A_X)\leq2/m$.  Define
\begin{align}
B_X=\cbr{\text{the algorithm intervenes on the node $P$ before
intervening on the node $X$}},
\end{align}
then using the law of total probability we can write
\begin{align}
\prob(A_X) = \prob(A_X|B_X)\prob(B_X)
+ \prob(A_X|B_X^c)\prob(B_X^c),
\end{align}
where $B_X^c$ stands for the complement of the event $B_X$, i.e.
\begin{align}
B_X^c=\cbr{\text{the algorithm intervenes on node $P$
after intervening on the node $X$}}.
\end{align}
Consider the probability $\prob(A_X|B_X)$. For this probability
not to be zero it must be the case that the node $X$ is a descendant
of the parent node $P$. Therefore, there must be a directed path
of length at least $m$ from $P$ to $X$. Note that any node on this
path except for the node $P$ cannot be intervened on before the node $X$
is intervened on. Therefore, by the time the algorithm intervenes
on the node $X$ there are at least $m$ nodes in the set of
candidate nodes from which it has to sample the node $X$ and hence
\begin{align}
\prob(A_X|B_X)\leq\frac{1}{m}.
\end{align}
Next, consider the probability $\prob(B_X^c)$. If there is a directed
path from $P$ to $X$, then no node on this path could have been
intervened on before the round at which the algorithm intervenes
on $X$ since otherwise $X$ would have been removed from the candidate
set.  Similarly, if there is a directed path from $X$ to $P$, then
intervening on any node on the path means excluding $X$ from the
candidate set. Thus, in these two cases $\prob(B_X^c)\leq1/(m+1)$
as there are at least $m+1$ nodes on the path of length at least
$m$. Lastly, consider the case when there are no directed paths
between $P$ and $X$. Since for this $X$ there exists a collider-free
path, there must be a path containing \emph{exactly one} ancestor
of both $X$ and $P$. Intervening on any node on this path other
than the node which is an ancestor of both $X$ and $P$ means excluding
$X$ from the candidate set because the intervened on node would be
an ancestor of $X$ or $P$ but not both. Thus, there are at least
$m$ nodes in the candidate set by the time the algorithm intervenes
on $X$ and therefore
\begin{align}
\prob(B_X^c)\leq\frac1{m}.
\end{align}
Bounding the other probabilities by one gives $\prob(A_X)\leq2/m$.
Bounding the number of nodes in the third group by $n$ and combining
all of the above we get
\begin{align}
\E[N(\bcG,P)]=\cO\pr{d^{m+1} + \frac{n}{m} + \frac{n}{\log_d(n)}}.
\end{align}
Finally, setting $m=\log_d\pr{\frac{n}{\log_d(n)}}-1$ finishes the proof.
\end{proof}

\section{Regret Bounds}\label{appendix:regret}
\subsection{Cumulative Regret}

\begin{lemma}\label{lemma:batch-size}
Let
\begin{align*}
A_{X,\bZ}&=\cbr{\exists x\in[K],\bz\in[K]^{\abs\bZ} :
\abs{\bar{R}-\bar{R}^{\doop(X=x,\bZ=\bz)}}>\Delta/2},\\
D_{X,Y,\bZ}&=\Bigl\{\exists
  x,y\in[K],\bz\in[K]^{\abs\bZ}
  : \abs{\hat{P}(Y=y)-\hat{P}(Y=y|\doop(X=x,\bZ=\bz))}
>\eps/2\Bigr\}
\end{align*}
for any $X,Y\in\cV$, $\bZ\subseteq\cP$ with nodes being the last in some
topological order of $\cP$ and $A_{X,\bZ}^c, D_{X,Y,\bZ}^c$ be their
compliments.  Define $E$ as the event that for every node we correctly
determine its descendants and whether it is an ancestor of $P$ using the
criteria described in \cref{sec:determine-ancestor}, i.e.
\begin{align*}
E=\bigcap_{\bZ}\bigcap_{X\in\cA(P)}
  A_{X,\bZ} \cap\bigcap_{X\in\cA^c(P)} A_{X,\bZ}^c
\cap\bigcap_{X\in\cV}\Bigl(\bigcap_{Y\in\bar\cD(X)}D_{X,Y,\bZ}
\cap\bigcap_{Y\in\cD^c(X)}D_{X,Y,\bZ}^c\Bigr),
\end{align*}
where the intersection with respect to $\bZ$ is over all possible sequences of
last elements of $\cP$ in all topological orders. Then it holds that
$\prob\cbr{E}\geq1-\delta$ if
$B=\max\cbr{\frac{32}{\Delta^2}\log\pr{\frac{8nK(K+1)^n}{\delta}},
\frac{8}{\eps^2}\log\pr{\frac{8n^2K^2(K+1)^n}{\delta}}}$.
\end{lemma}
\begin{proof}
By Hoeffding's inequality for bounded random variables for
fixed $X,Y\in\cV$ with $X\in\cA(Y)$, $\bZ\subseteq\cP$,
$x,y\in[K]$ and $\bz\in[K]^{\abs\bZ}$ it holds that
\begin{align}
&\abs{\hat{P}(Y=y|\doop(X=x,\bZ=\bz))-\prob\cbr{Y=y|\doop(X=x,\bZ=\bz)}}\geq\\
&\quad\geq\sqrt{\frac1{2B}\log\pr{\frac{8n^2K^2(K+1)^n}{\delta}}}
\end{align}
with probability at most $\frac{\delta}{4n^2K^2(K+1)^n}$ and
\begin{align}
\abs{\hat{P}(Y=y|\doop(\bZ=\bz))-\prob\cbr{Y=y|\doop(\bZ=\bz)}}
\geq\sqrt{\frac1{2B}\log\pr{\frac8\delta}}
\end{align}
with probability at most $\frac{\delta}{4nK(K+1)^n}$. Additionally, by
Hoeffding's inequality for $1$-subgaussian random variables
we have that for fixed $X\in\cV$ and $x\in[K]$ it holds that
\begin{align}
\abs{\E[R|\doop(X=x,\bZ=\bz)]-\bar{R}^{\doop(X=x,\bZ=\bz)}}
\geq\sqrt{\frac{2}{B}\log\pr{\frac{8nK(K+1)^n}{\delta}}}
\end{align}
with probability at most $\frac{\delta}{4nK(K+1)^n}$. Moreover,
\begin{align}
\abs{\bar{R}^{\doop(\bZ=\bz)}-\E[R|\doop(\bZ=\bz)]}
\geq\sqrt{\frac{2}{B}\log\pr{\frac{8(K+1)^n}{\delta}}}
\end{align}
with probability at most $\frac{\delta}{4(K+1)^n}$. Consider the event which is
the union of the above bad events. Since for $\bZ$ there are
$\sum_{\ell=0}^{\abs\cP}\binom{\abs\cP}{\ell}K^\ell=(K+1)^{\abs{\cP}}$ choices,
by union bound we have that the probability of this bad event is at most
$\delta$. Note that under the complement of this bad event for $X\not\in\cA(P)$
and all $x\in[K],\bz\in[K]^{\abs\bZ}$ by \cref{assumption:reward-gap} and the
choice of $B$ as in the statement of \cref{lemma:batch-size} we have
\begin{align}
\abs{\bar{R}^{\doop(\bZ=\bz)}-\bar{R}^{\doop(X=x,\bZ=\bz)}}&\leq
\abs{\bar{R}^{\doop(\bZ=\bz)}-\E[R|\doop(\bZ=\bz)]} \\
&\quad+\abs{\bar{R}^{\doop(X=x,\bZ=\bz)}-\E[R|\doop(X=x,\bZ=\bz)]} \\
&\leq\Delta/2,
\end{align}
and for some $x\in[K],\bz\in[K]^{\abs\bZ}$
\begin{align}
\abs{\bar{R}^{\doop(\bZ=\bz)}-\bar{R}^{\doop(X=x,\bZ=\bz)}}
&\geq\abs{\E[R|\doop(\bZ=\bz)] -\E[R|\doop(X=x,\bZ=\bz)]} \\
&\quad-\abs{\bar{R}^{\doop(\bZ=\bz)}-\E[R|\doop(\bZ=\bz)]} \\
&\quad-\abs{\E[R|\doop(X=x,\bZ=\bz) -\bar{R}^{\doop(X=x,\bZ=\bz)}]} \\
&>\Delta/2,
\end{align}
Similarly, under the complement of the same event we get that for
$Y\not\in\cD(X)$ and all $x,y\in[K],\bz\in[K]^{\abs\bZ}$ it holds that
\begin{align}
&\abs{\hat{P}(Y=y|\doop(\bZ=\bz))-\hat{P}(Y=y|\doop(X=x,\bZ=\bz)}\leq \\
&\quad\leq\abs{\hat{P}(Y=y|\doop(\bZ=\bz))-\prob\cbr{Y=y|\doop(\bZ=\bz)}} \\
&\quad\quad+\abs{\hat{P}(Y=y|\doop(X=x,\bZ=\bz))-\prob\cbr{Y=y|\doop(X=x,\bZ=\bz)}}\\
&\quad\leq\eps/2,
\end{align}
and if $Y\in\cD(X)$, then for some $x,y\in[K],\bz\in[K]^{\abs\bZ}$
\begin{align}
&\abs{\hat{P}(Y=y|\doop(\bZ=\bz))-\hat{P}(Y=y|\doop(X=x,\bZ=\bz)} \geq \\
&\quad\geq\abs{\prob\cbr{Y=y|\doop(X=x,\bZ=\bz)}
  -\prob\cbr{Y=y|\doop(\bZ=\bz)}}\\
&\quad\quad-\abs{\prob\cbr{Y=y|\doop(X=x,\bZ=\bz)}
  -\hat{P}(Y=y|\doop(X=x,\bZ=\bz))}\\
&\quad\quad-\abs{\hat{P}(Y=y|\doop(\bZ=\bz))
  -\prob\cbr{Y=y|\doop(\bZ=\bz)}}>\eps/2.
\end{align}

\end{proof}

\regret*
\begin{proof}
By lemma \cref{lemma:batch-size} it holds that for every node we can correctly
identify whether that node is an ancestor of $P$ and all the descendants of
that node with probability at least $1-\delta$ using the criteria described in
\cref{sec:determine-ancestor}.  That means that \cref{alg:thealg} will
correctly discover the parent node under the same good event in
$B\E[N(\bcG,P)]\sum_{\ell=1}^{\abs\cP+1}K^\ell\preceq
B\E[N(\bcG,P)]K^{\abs\cP+1}$ interventions since to discover each new parent we
need to perform interventions over all possible values of the previously
discovered parents and all the remaining candidate nodes.  Subsequently running
a standard bandit algorithm such as \acs{ucb} to find an optimal intervention
on $P$ leads to regret bound of $\sum_{\bx\in[K]^{\abs\cP}}\Delta_{\cP=\bx}\pr{
1+\frac{\log{T}}{\Delta_{\cP=\bx}^2}}$ \cite{lattimore2020bandit}.
\end{proof}

\subsection{Simple Regret}\label{sec:simple-regret}

In this subsection we provide an upper bound on simple regret. First,
similar to how it is done in the main text, we define conditional simple regret:
\begin{align}
r_\cL^T(\bcG,\cP|E)
=\max_{\bX\subseteq\cV}\max_{\bx\in[K]^{\abs{\bX}}}\E[R|\doop(\bX=\bx)]
-\E[R|\doop(\bX_{T+1}=\bx_{T+1}),E],
\end{align}
where $E$ is the event that all descendants of any node in $\cV$ are correctly
identified together with whether any node is an ancestor of $\cP$, defined in
\cref{lemma:batch-size}.
%Let the
Suppose that the learner runs a standard bandit algorithm that is designed to
minimize cumulative regret, for example, \acs{ucb}, from round $N(\bcG,\cP)+1$
to round $T$. After that the final intervention $\doop(X=x)$ for arbitrary
$x\in[K]$ and $X\in\cV$ is sampled with probability
\begin{align}
\frac1{T}\sum_{t=1}^{T}\I\cbr{I_t=\doop(X=x)},
\end{align}
where $I_t$ is the intervention performed in round $t$. Standard conversion of
cumulative regret bound to simple regret bound \citep[see e.g.,][Proposition
33.2]{lattimore2020bandit} leads to \cref{corollary:simple-regret} stated
below.
\begin{corollary}\label{corollary:simple-regret}
For the learner described above we can bound conditional simple regret
as follows:
\begin{align}
r_{\cL}^T(\bcG,\cP|E)
&\preceq\frac{1}{T}\max\cbr{\frac1{\Delta^2},\frac1{\eps^2}}
K^{\abs\cP+2}\E[N(\bcG,\cP)]\log\pr{\frac{nK^n}{\delta}} \\
&\quad+\frac1T\sum_{
  \bx\in[K]^{\abs\cP}}\Delta_{\cP=\bx}\pr{
    1+\frac{\log{T}}{\Delta_{\cP=\bx}^2}}.
\end{align}
\end{corollary}

\section{Generalization to Multiple Parent Nodes}\label{appendix:multiparent}

\begin{algorithm}[tb]
\caption{Full version of \ac{raps} for unknown number of parents}
\label{alg:full}
\begin{algorithmic}[1]
\Require{Set of nodes $\cV$ of $\bcG$, $\eps,\Delta$, probability
of incorrect parent set estimate $\delta$}
\Ensure{Estimated set of parent nodes $\hat\cP$}
\State $\hat\cP\gets\emptyset, \cS\gets\emptyset,
\hat{P}\gets\varnothing, \cC\gets\cV, \cD'\gets\emptyset$
\Comment{$\cD'$ is the set of descendants of last ancestor of $R$}
\State $B\gets\max\cbr{\frac{32}{\Delta^2}\log\pr{\frac{8nK(K+1)^n}{\delta}},
\frac8{\eps^2}\log\pr{\frac{8n^2K^2(K+1)^n}{\delta}}}$
\State Observe $B$ samples from \ac{pcm} and compute $\bar{R}$
and $\hat{P}(X)$ for all $X\in\cV$
\While{$\cC\neq\emptyset$}
\State $X\sim\Unif(\cC)$
\For{$x\in[K]$ and $\bz\in[K]^{\abs{\hat\cP}}$}
\State Perform $B$ interventions $\doop(X=x,\hat\cP=\bz)$
\State Compute $\bar{R}^{\doop(X=x,\hat\cP=\bz)},
\hat{P}(Y|\doop(X=x,\hat\cP=\bz))$ for all $Y\in\cV$
\EndFor
\State Estimate descendants of $X$:
\begin{align*}
\cD\gets\biggl\{Y\in\cV &\mid \exists x\in[K],\bz\in[K]^{\abs{\hat\cP}}, \\
&\text{ such that }\abs{\hat{P}(Y|\doop(\hat\cP=\bz))
  -\hat{P}(Y|\doop(X=x,\hat\cP=\bz))}>\eps/2\biggr\}
\end{align*}
\If{$\exists x\in[K],\bz\in[K]^{\abs{\hat\cP}}$ such that
$\abs{\bar{R}^{\doop(\hat\cP=\bz)}-\bar{R}^{\doop(X=x,\hat\cP=\bz)}}
>\Delta/2$}
\State $\cC\gets\cD\setminus\cbr{X}$
\State $\hat{P}\gets X, \cD'\gets\cD$
\Else
\State $\cC\gets\cC\setminus\cD$
\State $\cS\gets\cS\cup\cD$
\EndIf
\If {$\cC=\emptyset$ and $\hat{P}\neq\varnothing$}
\State $\hat{\cP}\gets\hat\cP\cup\cbr{\hat{P}}$
\State $\cS\gets\cS\cup\cD'$
\State $\cC\gets\cV\setminus\cS$
\State $\hat{P}\gets\varnothing$
\EndIf
\EndWhile
\State\Return $\hat\cP$
\end{algorithmic}
\end{algorithm}

\multiparent*
\begin{proof}
As noted in the main text, \cref{alg:multiparent} discovers the parent nodes in
a reverse topological order, let $\tau\in\btau(\cP)$ be such an order and
$i=\abs{\hat{\cP}}\leq\abs{\cP}$ be a number of the iteration of the while loop
in \cref{alg:multiparent}. First, assume $i<\abs\cP$. We argue that the
expected number of interventions during a call of \cref{alg:thealg} is the same
as the expected number of interventions done by \cref{alg:thealg} when there is
only one parent node $P$ which is equal to the $(\abs\cP-i)$-th element of
$\tau$ and $\hat{\cP}=\emptyset$ on the graph
$\bcG_{\cV\setminus\cS(\tau,\abs\cP-i)}$. If $i=\abs\cP$, then we need to show
that the call to \cref{alg:thealg} with $\hat\cP=\cP$ on the graph
$\bcG_{\cV\setminus\cD(\cP)}$ is the same as running \cref{alg:thealg} on the
graph-parent-node pair $(\bcG_{\cV\setminus\cD(\cP)},\varnothing)$ with
$\hat\cP=\emptyset$. Proving these results leads to the proof of the result of
the theorem since by the assumption of the theorem we have that
$(\bcG_{\cV\setminus\cS(\tau,\abs\cP-i)}, P)$ and
$(\bcG_{\cV\setminus\cD(\cP)},\varnothing)$ satisfy the assumptions of either
\cref{thm:log-condition} or \cref{thm:sublinear}.  Let $X$ be an arbitrary
node, intervened on during the call of \cref{alg:thealg} by
\cref{alg:multiparent}. If $X$ is an ancestor of $P$ in $\bcG_{\cC}$ for some
$\cC$, then $X$ is also an ancestor of $P$ in
$\bcG_{\cV\setminus\cS(\tau,\abs\cP-i)\cap\cC}$ since no ancestor of $P$ is
contained in $\cS(\tau,\abs\cP-i)$ because of its definition. Then in there
will be a recursive call for the candidate set $\bar\cD_{\cC}(X)$.  At the same
time, if $X$ is not an ancestor of any node in $\hat\cP$ in $\bcG_{\cC}$ then
$X$ is not an ancestor of any such node in
$\bcG_{\cV\setminus\cS(\abs\cP-i)\cap\cC}$ since it is a subgraph of the graph
$\bcG_{\cC}$, and therefore there will be a recursive call for the candidate
set $\cD_{\cC}^c(X)$.  Finally, if $X$ is not an ancestor of $P$ in
$\bcG_{\cC}$ but there exist some $P'\in\cP$ such that $P'\neq P$ and $X$ is an
ancestor of $P'$ in $\bcG_{\cC}$, then the call to \cref{alg:thealg} by
\cref{alg:multiparent} will return $P'$ which is a contradiction. Thus, by
induction on the elements of $\cP$ we have that the sequences of candidate sets
$\cC$ with which the recursive function of \cref{alg:thealg} is called when
this algorithm is called by \cref{alg:multiparent} and  the sequences of of
candidate sets $\cC$ with which the recursive function of \cref{alg:thealg} is
called when this algorithm is executed on
$\bcG_{\cV\setminus\cS(\tau,\abs\cP-i)}$ are equally likely and we conclude
that the expected number of interventions in these two cases is the same.
\end{proof}

\ermultiparent*
\begin{proof}
The minimum $p$ in the condition of \cref{corollary:erdos-renyi}
grows with decreasing $n$. Therefore, for the condition of
\cref{thm:multiparent} it suffices that for the smallest subgraph
$\bcG_{\cV\setminus\cD(\cP)}$ considered by \cref{alg:multiparent}
the condition of \cref{corollary:erdos-renyi} holds. However, this
graph needs to be of size at least $c_1\log^k(n)$ since otherwise
the bound is trivial. Plugging $n'=c_1\log^k(n)$ as $n$ in the
condition of \cref{corollary:erdos-renyi} leads to the desired
result.
\end{proof}

\section{Universal Lower Bound}\label{appendix:lower-bound}

% \section{Universal Lower Bound}\label{sec:lower-bound}

In this section we show that the result of \cref{sec:exact}
is tight in the sense that any algorithm that finds the parent
node $P$ (or determines it does not exist in the graph $\bcG$)
requires at least the number of interventions performed by \acs{raps}.

\begin{restatable}{theorem}{lowerbound}\label{thm:lower-bound}
Fix a causal graph $\bcG$ and a parent node $P$. Any learner $\cL$ that
correctly identifies the parent node $P$, for any graph obtained from $\bcG$ by
relabeling\footnote{Relableing corresponds to the assignment of indices 1
through $n$ identifying each node, but not the assignment of random variables
to nodes.}of the nodes and having $P$ take one of $n$ vertices or
$P=\varnothing$, satisfies
\begin{equation*}
\E[N_{\cL}(\bcG,P)]\geq
\sum_{X\in\cV}\frac1{
\abs{\cA_{\bcG}(P)\triangle\cA_{\bcG}(X)\setminus\cbr{X}}+1},
\end{equation*}
where the expectation is taken with respect to the random assignment of the
indices 1 through $n$ identifying each node and the randomness in running
\acs{raps}.
\end{restatable}

The proof could be found below. Consider, for example, a null graph
$\bcG=(\cV,\emptyset)$. The number of ancestors of every vertex is equal to one
and the expression in \cref{thm:lower-bound} becomes $\Omega(n)$. At the same
time, even if the graph is connected and its skeleton is a line graph it is
possible to have a lower bound of $\Omega(n)$ by having all vertices separated
from $P$ by colliders as in \cref{fig:colliders}.  In this figure, every node
$X_i$ where $1\leq i<n-1$ is odd is a collider on the path between $X_{i-1}$
and $X_{i+1}$ and assume that $X_0\equiv P$. The number of ancestors of every
node $X_j$, where $1<j<n-1$ is even, equals one leading to $\Omega(n)$ lower
bound. If a graph is connected and has no colliders, \cref{thm:sublinear}
results in $\cO\pr{\frac{n}{\log_d(n)}}$ upper bound on the number of
interventions. The upper bound in \cref{thm:sublinear} is tight for perfect
$d$-ary trees. In such trees the number of non-common ancestors between $P$
(possibly $P=\varnothing$) and any node $X$ is lower bounded by the distance
from $X$ to the root assuming that $X$ comes from one of the subtrees, other
than the subtree containing $P$. Thus, considering only the last term in the
summation results in
\begin{align*}
\E[N_{\cL}(\bcG,P)]&\geq
\sum_{h=1}^{\log_d(n+1)}\frac{(d-1)d^{h-1}}{h+1}
\geq\frac{(n+1)(d-1)}{d(\log_d(n+1)+1)},
\end{align*}
which matches the asymptotic upper bound of \cref{thm:sublinear}.
By adding an extra knowledge about the essential graph, the algorithm in \citet{greenewald2019sample} can detect the parent node with at most  $\cO(\log n)$ number of atomic interventions.

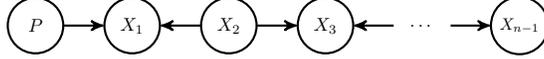
\begin{figure}
\centering
\begin{tikzpicture}[->,every node/.style=node,node distance=1.5em]
\node (p) {$P$};
\node[right=of p] (x1) {$X_1$};
\node[right=of x1] (x2) {$X_2$};
\node[right=of x2] (x3) {$X_3$};
\node[right=of x3,draw=none] (dots) {$\dots$};
\node[right=of dots,scale=0.9] (xnm1) {$X_{n-1}$};
\path
(p) edge (x1)
(x2) edge (x1)
(x2) edge (x3)
(dots) edge (x3)
(dots) edge (xnm1);
\end{tikzpicture}
\caption{An example of a graph with a skeleton that is a line graph
and $n/2-1$ colliders ($n$ is assumed to be even). Our lower bound in \cref{thm:lower-bound} implies that any learner requires $\Omega(n)$
atomic interventions to discover $P$ in this graph.}\label{fig:colliders}
\end{figure}

\begin{proof}[Proof of \cref{thm:lower-bound}]
The proof uses Yao's principle \citep{yao1977probabilistic} from
which it follows that we need to show that the best deterministic
algorithm performs at least the number of interventions in the lower
bound of the theorem against some distribution over graphs $\bcG$.
Thus, in what follows, let $\prob$ be the probability measure over
the random graphs obtained by randomly and uniformly relabeling the
nodes of the graph $\bcG$, such a random graph will be denoted by
$\bcH$.  Assume that the learner $\cL$ does not intervene on the
same node twice. Moreover, assume that if it intervenes on a
non-ancestor of $P$ it will not subsequently intervene on any of
its' descendants and that if it intervenes on an ancestor of $P$
it will not subsequently intervene on any of the non-descendants
of that ancestor.  We can make this assumption as any learner which
does not satisfy it would intervene on the same nodes with the same
results as a new learner which avoids making these redundant
interventions. We denote by $\bcD$ the set of all learners that
satisfy this assumption. Our goal is to lower bound
\begin{align}
\inf_{\cL\in\bcD}\E_{\bcH}[N(\bcH,\cL)]
=\inf_{\cL\in\bcD}\E\br{\sum_{X\in\cV}\I\cbr{A_X}}
=\inf_{\cL\in\bcD}\sum_{X\in\cV}\prob\cbr{A_X},
\end{align}
where $A_X=\cbr{\text{the learner $\cL$ intervenes on the node $X$}}$.
Let $Z$ be a node selected uniformly at random and independently
from the sampling of the graph and the parent node, then
\begin{align}
\sum_{X\in\cV}\prob\cbr{A_X}
=n\sum_{X\in\cV}\prob\cbr{Z=X}\prob\cbr{A_X}
=n\prob\cbr{A},
\end{align}
where $A=\cbr{\text{the learner $\cL$ intervenes on a randomly
selected node $Z$}}$. Note that for learner $\cL$ there are only
two ways not to intervene on any node $Z$. The first is to intervene
either on an ancestor of $Z$ which is not an ancestor of $P$ and
the second is to intervene on an ancestor of $P$ which is not an
ancestor of $Z$. Using this we get
\begin{align}\label{eq:lower-bound-prob}
\prob\cbr{A}=\prob\cbr{\text{$\cL$ intervenes on $Z$ before
any node in $\cA_{\bcH}(P)\triangle\cA_{\bcH}(Z)\setminus\cbr{Z}$}}.
\end{align}
We note that any deterministic learner $\cL$ could be represented
by a sequence of nodes $W_1,\dots,W_n$ with $W_i\neq W_j$ for $i\neq
j$, and for a random graph $\bcH$ the learner intervenes on the node $W_i$ if there is
no element of $\cA_{\bcH}(P)\triangle\cA_{\bcH}(W_i)$ in the sequence
$\bW_{<i}=(W_1,\dots,W_{i-1})$. Using the law of total probability
we write
\begin{align}
\prob\cbr{A}&=\sum_{l=0}^{n-1}
\prob\cbr{A\;\big|\abs{\cA_{\bcH}(P)\triangle\cA_{\bcH}(Z)\setminus\cbr{Z}}=l}
\prob\cbr{\abs{\cA_{\bcH}(P)\triangle\cA_{\bcH}(Z)\setminus\cbr{Z}}=l}.
\end{align}
Moreover, we have
\begin{align}
&\prob\cbr{
A\;\big|\abs{\cA_{\bcH}(P)\triangle\cA_{\bcH}(Z)\setminus\cbr{Z}}=l}= \\
&\quad=\sum_{i=1}^{n-l}
\prob\Bigl\{\pr{
\cA_{\bcH}(P)\triangle\cA_{\bcH}(Z)\setminus\cbr{Z}}\cap\bW_{<i}=\emptyset|\\
&\qquad\qquad\qquad\abs{
\cA_{\bcH}(P)\triangle\cA_{\bcH}(Z)\setminus\cbr{Z}}=l,Z=W_i\Bigr\}
\times\prob\cbr{Z=W_i}\\
&\quad=\sum_{i=1}^{n-l}\frac{\binom{n-i}{l}l!(n-l)!}{(n-1)!}
\cdot\frac1n
=\frac1{l+1},
\end{align}
where to get to the last line we used the fact that $Z$ is selected uniformly
at random, we have to choose $l$ elements from $n-i$ elements to label the set
$\cA_{\bcH}(P)\triangle\cA_{\bcH}(Z)\setminus\cbr{Z}$ while the rest of the
nodes could be labeled randomly, and to get the last equality we used the
hockey-stick identity~\cite{jones1994generalized} similar to the proof of
\cref{thm:exact} in \cref{appendix:exact}. Finally, the result follows from
noticing that the probability of
$\abs{\cA_{\bcH}(P)\triangle\cA_{\bcH}(Z)\setminus\cbr{Z}}=l$ is independent of
the labeling of the nodes and thus is equal to $\frac{\abs{\cbr{X\in\cV :
\abs{\cA_{\bcG}(P)\triangle\cA_{\bcG}(X)\setminus\cbr{X}}=l}}}{n}$.
\end{proof}

\section{Experiments}\label{appendix:experiments}

In this section we discuss additional experimental results aimed at testing our
theoretical findings. Unless stated otherwise we obtain the average regret or
number of interventions to discover the parent node(s) and the standard
deviation over 20 independent runs.

\begin{figure*}
\centering
\begin{subfigure}[T]{0.21\linewidth}
\includegraphics[width=\linewidth]{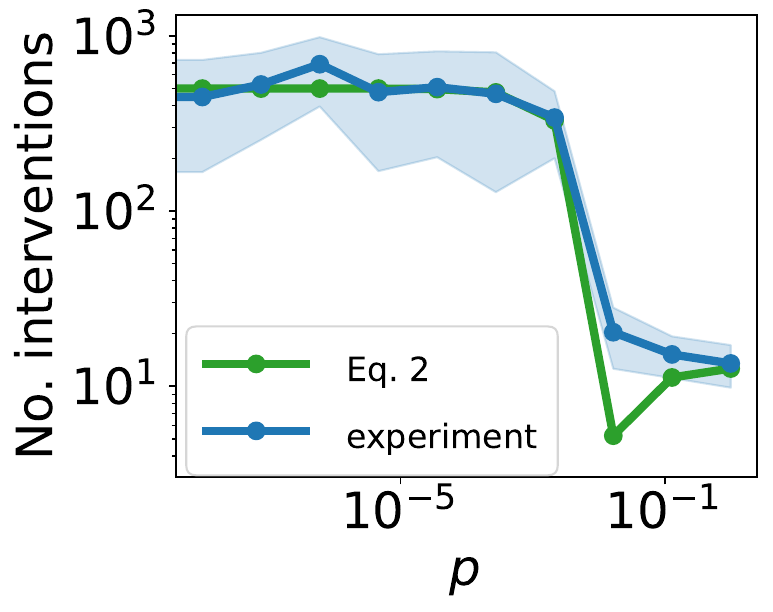}
\caption{}\label{fig:exact}
\end{subfigure}
% \begin{subfigure}[T]{0.1\linewidth}
% \includegraphics[width=\linewidth]{er-legend.pdf}
% \end{subfigure}
\begin{subfigure}[T]{0.21\linewidth}
\includegraphics[width=\linewidth]{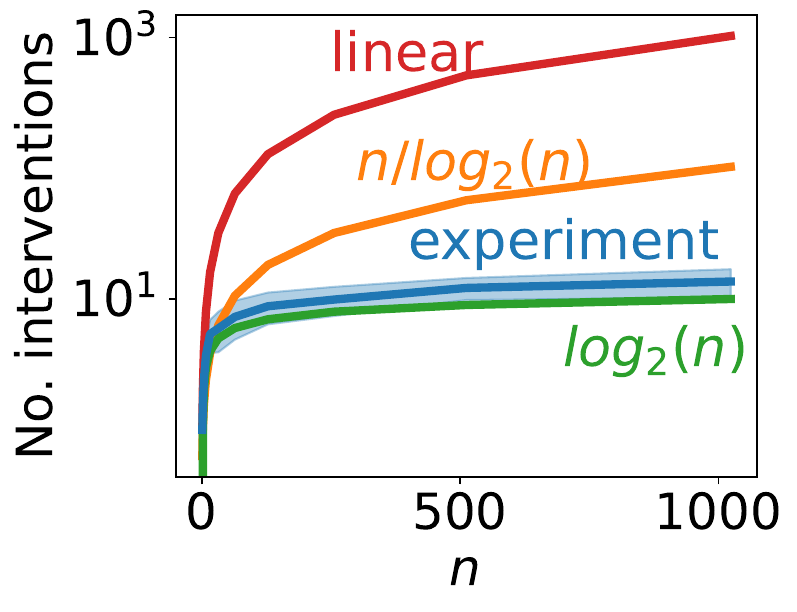}
\caption{}\label{fig:er-fast}
\end{subfigure}
\begin{subfigure}[T]{0.21\linewidth}
\includegraphics[width=\linewidth]{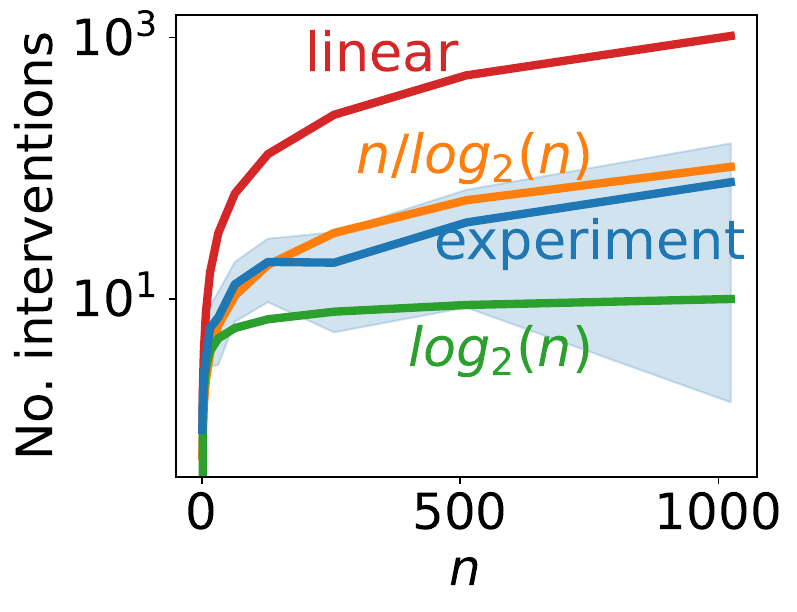}
\caption{}\label{fig:er-slow}
\end{subfigure}
\begin{subfigure}[T]{0.21\linewidth}
\includegraphics[width=\linewidth]{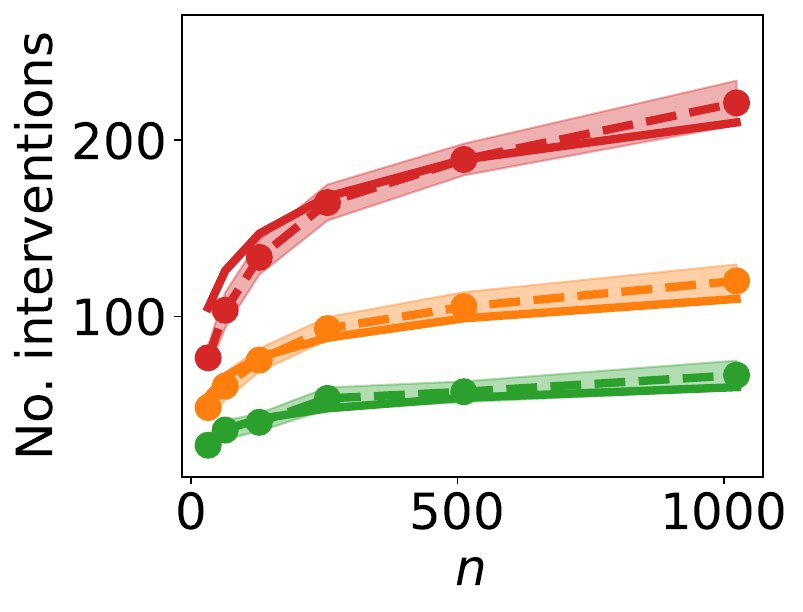}
\caption{}\label{fig:ermultiparent}
\end{subfigure}
\begin{subfigure}[T]{0.1\linewidth}
\includegraphics[width=\linewidth]{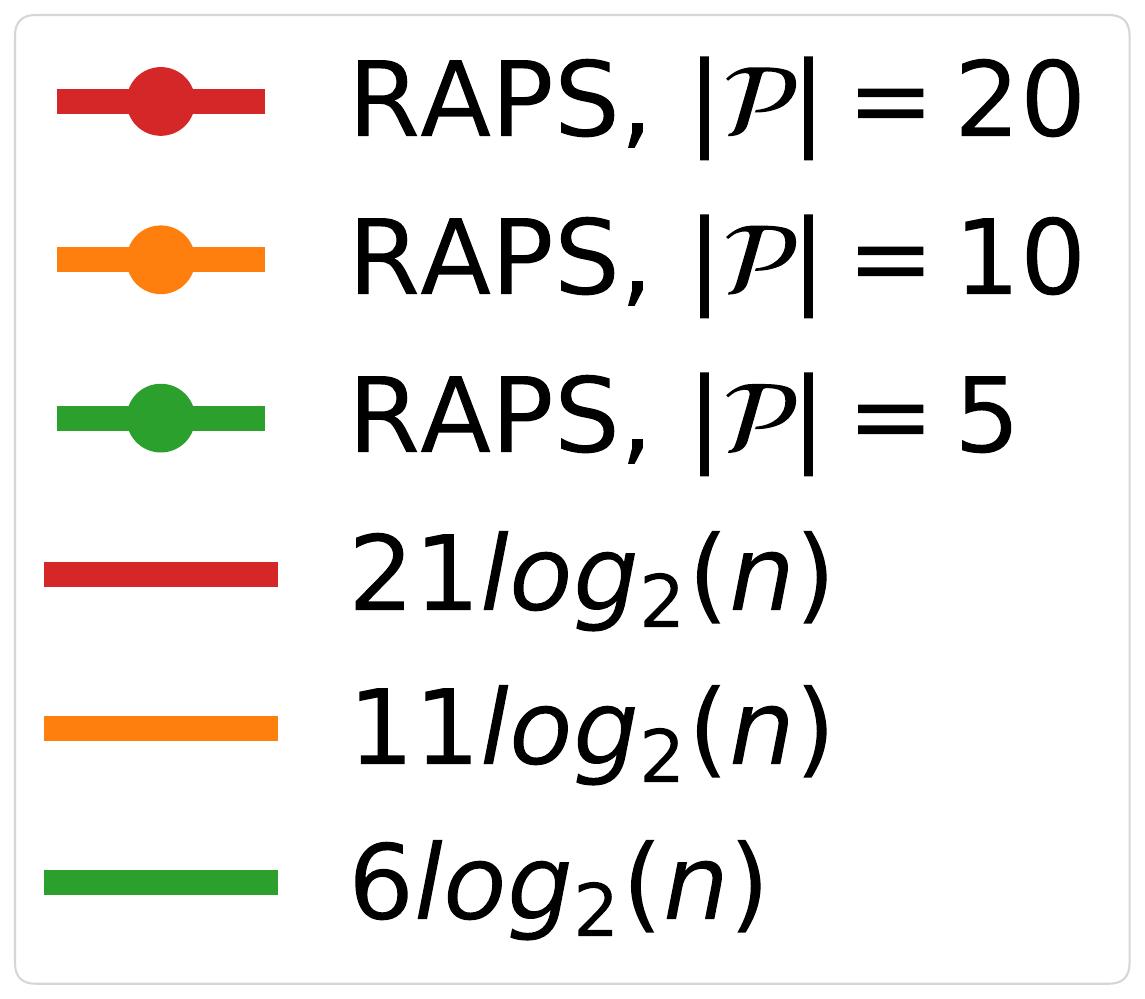}
\end{subfigure}
\caption{(a) Comparison between \cref{eq:exact} from \cref{thm:exact} and
the experimental number of interventions of \cref{alg:thealg} on
Erd\H{o}s-R\'enyi random \acp{dag}. (b-c) The results of running \ac{raps} on
Erd\H{o}s-R\'enyi random \acp{dag} with large and small $p$. (d) Results of
running \cref{alg:multiparent} to discover multiple parent nodes.}
\end{figure*}

% Firstly, we confirm that binary trees require
% $\Theta\pr{\frac{n}{\log(n)}}$ interventions to discover the parent
% node in \cref{fig:tree} where we also plot the two other regimes:
% the $\cO(\log n)$ ``fast`` regime of parent discovery and the trivial
% linear bound on the number of the interventions.

% \begin{figure}[b]
% \centering
% \includegraphics[width=0.75\linewidth]{tree.pdf}
% \caption{The results of running \ac{raps} on perfect
% binary tree conforming $\Theta\pr{\frac{n}{\log_d(n)}}$
% number of interventions.}\label{fig:tree}
% \end{figure}

\subsection{Algorithmic Aspect}

In this subsection we ignore the statistical aspect of finding the set of
parent nodes and assume that each intervention gives us perfect information
about the descendants of the intervened on node as well as whether that node is
an ancestor of a parent node.

Firstly, we confirm that the \cref{eq:exact} could be used to compute the
expected number of interventions required to discover the parent node in
\cref{fig:exact}. In this experiment we generate Erd\H{o}s-R\'enyi random
\ac{dag} as described in \cref{sec:erdos-renyi} with different values of $p$
and for each such \ac{dag} compute the expected number of interventions as
predicted in \cref{eq:exact}, as well as perform 20 independent runs on the
same graph of \cref{alg:thealg}. The average over those 20 runs and the
standard deviation are shown as the line and the shaded area, respectively,
similar to the other figures. For this experiment we set the number of nodes in
all graphs $n=1000$. Notice also that \cref{eq:exact} matches the lower bound
in \cref{thm:lower-bound}, thus in \cref{fig:exact} we show that the
performance of our algorithm matches the performance of the best possible
algorithm.

Secondly, in \cref{fig:er-fast} we confirm the result of
\cref{corollary:erdos-renyi} by showing that when
$p=1-\pr{\frac{0.5}{\log_2(n)-1}}^{1/(\log_2(n)-1)}$ in Erd\H{o}s-R\'enyi
random \ac{dag} obtained as discussed in \cref{sec:erdos-renyi}, then the
number of interventions required scales as $\cO(\log n)$.  At the same time, in
\cref{fig:er-slow} we show that when $p=\frac{\log{n}}{n}$, then the number of
interventions scales as $\frac{n}{\log{n}}$.

% \begin{figure}
% \centering
% \includegraphics[width=0.75\linewidth]{erdos-renyi-fast.pdf}
% \caption{The results of running \ac{raps} on Erd\H{o}s-R\'enyi
% random \acp{dag} with $p=1-\pr{\frac{0.5}{\log_2(n)-1}}^{1/(\log_2(n)-1)}$.
% The number of interventions is asymptotically logarithmic as predicted
% by \cref{corollary:erdos-renyi}}\label{fig:er-fast}
% \end{figure}

% \begin{figure}
% \centering
% \includegraphics[width=0.75\linewidth]{erdos-renyi-slow.pdf}
% \caption{The results of running \ac{raps} on Erd\H{o}s-R\'enyi
% random \ac{dag} with $p=\frac{\log n}{n}$. We observe that the number
% of interventions is on the order of $\frac{n}{\log n}$.}\label{fig:er-slow}
% \end{figure}

Additionally, we verify the result of \cref{corollary:ermultiparent} in
\cref{fig:ermultiparent}. On this figure we see that in Erd\H{o}s-R\'enyi
graphs with $p$ as in the lower bound of \cref{corollary:ermultiparent} (with
$c_0=0.5,c_1=1$ and $k=1$) the number of interventions required to discover
$\abs\cP$ parents grows as $(\abs\cP+1)\log(n)$.

\subsection{Statistical Aspect}

In \cref{fig:erdos-renyi-regret} we present the regret of running our approach
\acs{raps}+\acs{ucb} and of running just \acs{ucb} on Erd\H{o}s-R\'enyi graphs.
For this figure we set $p=\log\log(10)/\log(10)$ and $K=4$. We sample an
Erd\H{o}s-R\'enyi graph as discussed in \cref{sec:erdos-renyi} with $9$ nodes
and then add the reward node with a uniformly selected parent. The \acs{pcm} is
such that each node takes the value of a randomly selected parent and uniformly
sampled value in the set $[K]$ when there are no parents. The probability of
the reward node taking value equal to 1 is the value of its' parent divided by
the maximum value that it can take, $K$. The value of $\delta$ is set to be
$0.01$. For \acs{ucb} algorithm there is only one line since the regrets
between the different runs are very close. We can see that while the average
regrets of the two approaches are close, \acs{raps}+\acs{ucb} has a much bigger
variance: it can be much faster than \acs{ucb} due to quickly finding the
parent node or it can not find the parent node during the selected horizon
$T=10^7$. In our analysis of the results we found that this is due to the large
budget $B$ that might be required for some Erd\H{o}s-R\'enyi graphs.  Note that
with our approach it is possible to estimate the number of steps it will take
to discover the set of parent nodes and thus one can decide whether or not to
use \acs{raps} before the experiment.  Due to the dependence of budget $B$ on
$\eps$ and $\Delta$, we explore how the values of these variables depend on the
parameters of Erd\H{o}s-R\'enyi graphs and the number of parents.
\begin{figure}
\centering
\includegraphics[width=0.3\linewidth]{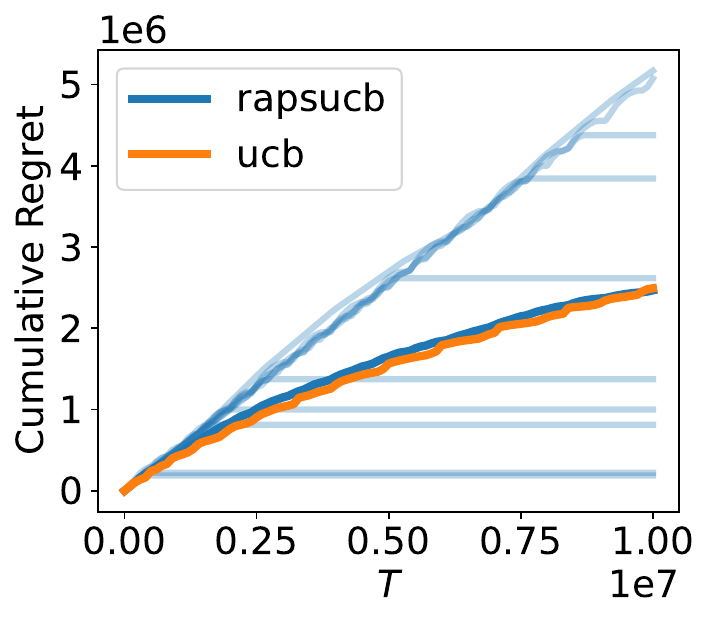}
\caption{Regret on Erd\H{o}s-R\'enyi graph with $p=\log\log(n)/\log(n)$.
Bold opaque lines show the mean over 10 runs, other lines show the regrets
for each of the 10 runs.}\label{fig:erdos-renyi-regret}
\end{figure}

In \cref{fig:epsgap} we present the results of our experiments aimed at
studying the behavior of $\eps$ and $\Delta$.  For the first figure we vary
probability of an edge $p$ and set $n=10$, for the second we vary the number of
nodes $n$ and set probability of an edge to be $p=\log\log(n)/\log(n)$, while
for the last figure we set $p=\log\log\log(n)/\log\log(n)$ and $n=16$. We can
see that for Erd\H{o}s-R\'enyi graphs $\eps$ can take small values for
substantially high probability and its' value decreases with the number of
nodes $n$.
\begin{figure}
\centering
\begin{subfigure}[T]{0.3\linewidth}
\includegraphics[width=\linewidth]{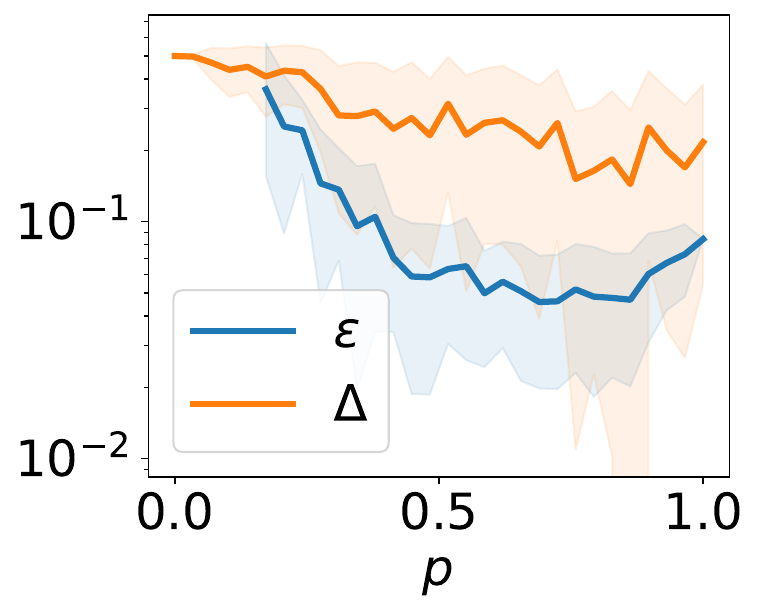}
\end{subfigure}
\begin{subfigure}[T]{0.3\linewidth}
\includegraphics[width=\linewidth]{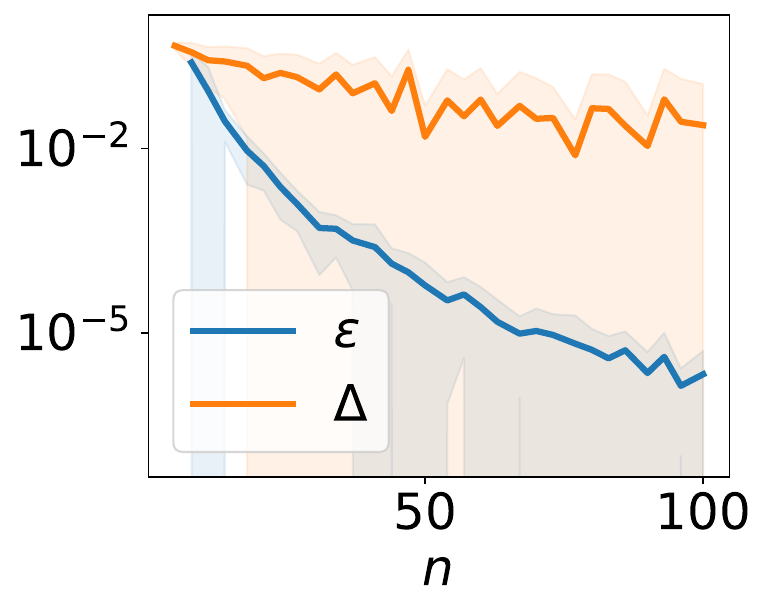}
\end{subfigure}
\begin{subfigure}[T]{0.3\linewidth}
\includegraphics[width=\linewidth]{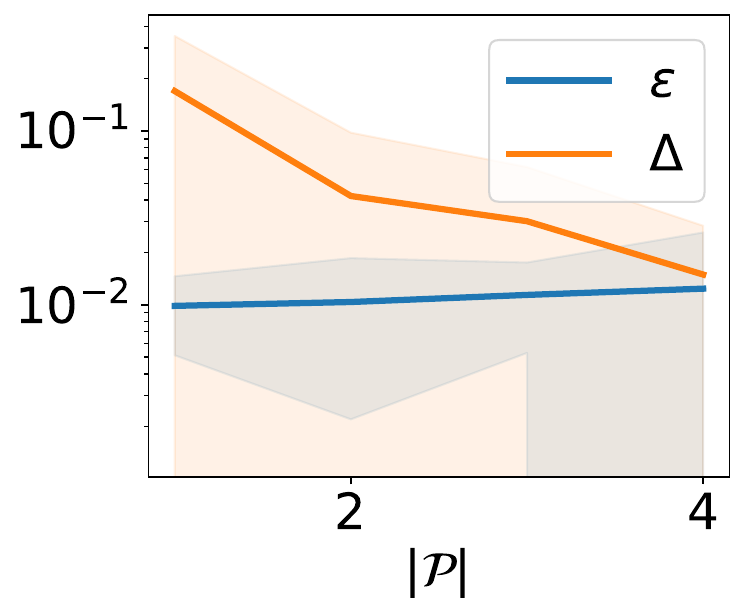}
\end{subfigure}
\caption{Values of $\eps$ and $\Delta$ for different parameters
of Erd\H{o}s-R\'enyi graphs and the number of parents.}\label{fig:epsgap}
\end{figure}

% \begin{figure}
% \centering\includegraphics[width=0.75\linewidth]{multiparent.pdf}
% \caption{Simulation of running \cref{alg:multiparent} on Erd\H{o}s-R\'enyi
% random graph with $p$ equal to the lower bound in
% \cref{corollary:ermultiparent} with
% $c_0=0.5,c_1=1,k=1$.}\label{fig:ermultiparent}
% \end{figure}

\end{document}